\documentclass[sigconf]{acmart}

\AtBeginDocument{%
  \providecommand\BibTeX{{%
    \normalfont B\kern-0.5em{\scshape i\kern-0.25em b}\kern-0.8em\TeX}}}

\copyrightyear{2022}
\acmYear{2022}
\setcopyright{acmcopyright}
\acmConference[WWW '22] {Proceedings of the ACM Web Conference 2022}{April 25--29, 2022}{Virtual Event, Lyon, France.}
\acmBooktitle{Proceedings of the ACM Web Conference 2022 (WWW '22), April 25--29, 2022, Virtual Event, Lyon, France}
\acmPrice{15.00}
\acmISBN{978-1-4503-9096-5/22/04}
\acmDOI{10.1145/3485447.3512068}

\usepackage{graphicx}
\usepackage{subfigure}
\usepackage{multirow}
\usepackage[ruled, vlined]{algorithm2e}
\usepackage{array}
\usepackage{bm}
\usepackage{amsmath}
\usepackage{enumitem}
\usepackage{amsthm}

\newtheorem{theorem}{Theorem}[section]

\begin{document}

\title[Fast Variational AutoEncoder with Inverted Multi-Index]{Fast Variational AutoEncoder with Inverted Multi-Index for Collaborative Filtering}

\author{Jin Chen}
\authornote{This work was done when the author Jin Chen was at University of Science and Technology of China for intern.}
\affiliation{
  \country{University of Electronic Science and Technology of China}
  \city{University of Science and Technology of China}
  }
\email{chenjin@std.uestc.edu.cn}

\author{Defu Lian}
\authornote{Corresponding author}
\affiliation{
	\country{University of Science and Technology of China}}
\email{liandefu@ustc.edu.cn}

\author{Binbin Jin}
\affiliation{
  \country{Huawei Cloud Computing Technologies Co., Ltd.}}
\email{jinbinbin1@huawei.com}

\author{Xu Huang}
\affiliation{
	\country{University of Science and Technology of China}}
\email{angus_huang@mail.ustc.edu.cn}

\author{Kai Zheng}
\affiliation{
	\country{University of Electronic Science and Technology of China}}
\email{zhengkai@uestc.edu.cn}

\author{Enhong Chen}
\affiliation{
	\country{University of Science and Technology of China}}
\email{cheneh@ustc.edu.cn}

\begin{abstract}
  Variational AutoEncoder (VAE) has been extended as a representative nonlinear method for collaborative filtering. However, the bottleneck of VAE lies in the softmax computation over all items, such that it takes linear costs in the number of items to compute the loss and gradient for optimization. This hinders the practical use due to millions of items in real-world scenarios. Importance sampling is an effective approximation method, based on which the sampled softmax has been derived. However, existing methods usually exploit the uniform or popularity sampler as proposal distributions, leading to a large bias of gradient estimation. To this end, we propose to decompose the inner-product-based softmax probability based on the inverted multi-index, leading to sublinear-time and highly accurate sampling. Based on the proposed proposals, we develop a fast Variational AutoEncoder (FastVAE) for collaborative filtering. FastVAE can outperform the state-of-the-art baselines in terms of both sampling quality and efficiency according to the experiments on three real-world datasets.
\end{abstract}



\renewcommand{\shortauthors}{Jin Chen, Defu Lian, Binbin Jin, Xu Huang, et al.}

\maketitle
\section{Introduction}
\label{sec:introduction}

Recommendation techniques play a role in information filtering to address the information overload in the era of big data. After decades of development, recommendation techniques have shifted from the latent linear models to deep non-linear models for modeling side features and feature interactions among sparse features. Variational AutoEncoder~\cite{Kingma2014} has been extended as a representative nonlinear method (Mult-VAE) for recommendation~\cite{liang2018variational}, and received much attention among the recommender system community in recent years~\cite{rakesh2019linked, sachdeva2019sequential,shenbin2020recvae,tang2019correlated,yu2019vaegan}. 
Mult-VAE encodes each user's observed data with a Gaussian-distributed latent factor and decodes it to a probability distribution over all items, which is assumed a softmax of the inner-product-based logits. Mult-VAE then exploits multinomial likelihood as the objective function for optimization, which has been proved to be more tailored for implicit feedback than the Gaussian and logistic likelihood. 

However, the bottleneck of Mult-VAE lies in the log-partition function over the logits of all items in the multinomial likelihood. The time to compute the loss and gradient in each training step grows linearly with the number of items. When there are an extremely large number of items, the training of Mult-VAE is time-consuming, making it impractical in real recommendation scenarios. To address this problem, self-normalized importance sampling is used for approximation~\cite{bengio2008adaptive,jean2015using} since the exact gradient involves computing expectation with respect to the softmax distribution. The approximation of the exact gradient leads to the efficient sampled softmax, but it does not converge to the same loss as the softmax. The only way to eliminate the bias is to treat the softmax distribution as the proposal distribution, but it is not efficient.

In spite of the well-known importance of a good proposal, many existing methods still often use simple and static distributions, like uniform or popularity-based distribution~\cite{mikolov2013efficient}. The problem of these proposals lies in large divergence from the softmax distribution, so that they need a large number of samples to achieve a low bias of gradient. The recent important method is to use quadratic kernel-based distributions~\cite{blanc2018adaptive} as the proposal, which are not only closer to the softmax distribution, but also efficient to sample from. However, the quadratic kernel is not always a good approximation of the softmax distribution, and it suffers from a large memory footprint due to the feature mapping of the quadratic kernel. 

Recently maximum inner product search (MIPs) algorithms have been widely used for fast top-k recommendation with low accuracy degradation~\cite{morozov2018non,ram2012maximum,shrivastava2014asymmetric}, but they always return the same results to the same query so that they can not be directly applied for item sampling. On this account, the MIPs indexes have been constructed over the randomly perturbed database for probabilistic inference in log-linear models and become a feasible solution to sample from the softmax distribution~\cite{mussmann2016learning}. However, this not only increases both data dimension and sample size, but also makes the samples correlated. Moreover, this may also require to rebuild the MIPs index from scratch once the model gets updated, which has a significant impact on the training efficiency. Therefore, 
it is necessary to design sampling algorithms tailored for the MIPs indexes.

To this end, based on the popular MIPs index -- inverted multi-index~\cite{babenko2014inverted}, we propose a series of proposal distributions, from which items can be efficiently yet independently sampled, to approximate the softmax distribution. The basic idea is to decompose item sampling into multiple stages. In each except the last stage, only a cluster index is sampled given the previously sampled clusters. In the last stage, items being simultaneously assigned to these sampled clusters are sampled according to uniform, popularity, or residual softmax distribution. Since there are a few items left, items are sampled from these approximated distributions in sublinear or even constant time. In some cases, the decomposed sampling is as exact as sampling from softmax, such that the quality of sampled items can be guaranteed. These samplers are then adopted to efficiently train Variational AutoEncoder for collaborative filtering (FastVAE for short). FastVAE\footnote{The code is released in https://github.com/HERECJ/FastVae\_Gpu} is evaluated extensively on three real-world datasets, demonstrating that FastVAE outperforms the state-of-the-art baselines in terms of sampling quality and efficiency.

The contributions can be summarized as follows:
\begin{itemize}[leftmargin=*]
    \item To the best of our knowledge, we discover high-quality approximated softmax distributions for the first time, by decomposing the softmax probability based on the inverted multi-index.
    \item We design an efficient sampling process for these approximate softmax distributions, from which items can be independently sampled in sublinear or even constant time. These samplers are applied for developing the fast Variational AutoEncoder.
    \item We evaluate extensively the proposed algorithms on four real-world datasets, demonstrating that FastVAE performs at least as well as VAE for recommendation. Moreover, the proposed samplers are highly accurate compared to existing sampling methods, and perform sampling with high efficiency.
\end{itemize}

\section{Related Work}
In this paper, we propose a new series of approximate softmax distributions based on the inverted multi-index. These samplers are used for efficiently training VAE. Therefore, 
we mainly survey related work about efficient softmax, negative sampling and maximum inner product search. 
Please refer to the survey~\cite{zhang2019deep,wang2019sequential} for deep learning-based recommender systems, and the survey~\cite{adomavicius2005toward} for classical recommendation algorithms.

\subsection{Efficient Softmax Training}
Sampled softmax improves training based on self-normalized importance sampling~\cite{bengio2008adaptive} with a mixture proposal of unigram, bigram and trigrams. Hierarchical softmax~\cite{morin2005hierarchical} uses the tree structure and lightRNN~\cite{li2016lightrnn} uses the table to decompose the softmax probability such that the probability can be quickly computed. Noise-Contrastive Estimation~\cite{gutmann2010noise} uses nonlinear logistic regression to distinguish the observed data from some artificially generated noise, and has been successfully used for language modeling~\cite{mnih2013learning}. Sphere softmax~\cite{de2015exploration,vincent2015efficient} replaces the exponential function with a quadratic function, allowing exact yet efficient gradient computation. 

\subsection{Negative Sampling in RS}
Dynamic negative sampling (DNS)~\cite{zhang2013optimizing} draws a set of negative samples from the uniform distribution and then picks the item with the largest prediction score. Similar to DNS, the self-adversarial negative sampling~\cite{sun2019rotate} draws negative samples from the uniform distribution but treats the sampling probability as their weights. Kernel-based sampling~\cite{blanc2018adaptive} picks samples proportionally to a quadratic kernel, making it fast to compute the partition function in the kernel space and to sample entries in a divide and conquer way. Locality Sensitive Hashing (LSH) over randomly perturbed databases enable sublinear time sampling~\cite{mussmann2016learning} and LSH itself can generate correlated and unnormalized samples~\cite{spring2017new}, which allows efficient estimation of the partition function. Self-Contrast Estimator~\cite{goodfellow2014distinguishability} copies the model and uses it as the noise distribution after every step of learning. Generative Adversarial Networks~\cite{wang2017irgan,jin2020sampling} directly learn the noise distribution via the generator networks. 

\subsection{Maximum Inner Product Search}
The MIPS problem is challenging since the inner product violates the basic axioms of a metric, such as a triangle inequality and non-negative. Some methods try to transform MIPS to nearest neighbor search (NNS) approximately~\cite{shrivastava2014asymmetric} or exactly~\cite{bachrach2014speeding,neyshabur2015symmetric}. The key idea of the transformation lies in augmenting database vectors to ensure them an (nearly) identical norm, since MIPS is equivalent to NNS when the database vectors are of the same norm. After the transformation, a bulk of algorithms can be applied for ANN search, such as Euclidean Locality-Sensitive Hashing~\cite{datar2004locality}, Signed Random Projection~\cite{shrivastava2014improved} and PCA-Tree~\cite{bachrach2014speeding}. Several existing work also studies quantization-based MIPS by exploiting additive nature of inner product, such as additive quantization~\cite{blanc2018adaptive}, composite quantization~\cite{zhang2014composite} and even extends PQ from the Euclidean distance to the inner product~\cite{guo2016quantization}. Similarly, the graph-based index has been extended to MIPS~\cite{morozov2018non}, achieving remarkable performance.

\section{Preliminaries}
\subsection{Mult-VAE}\label{sec:vae-cf}
Assuming recommender models operate on $N$ users' implicit behavior (e.g. click or view) over $M$ items, where each user $u$ is represented by the observed data $\bm{y}_u$ of dimension $M$. Each entry $y_{ui}$ indicates an interaction record to the item $i$, where $y_{ui}=0$ indicates no interaction. Mult-VAE~\cite{liang2018variational} is a representative nonlinear recommender method for modeling such implicit data. It particularly encodes $\bm{y}_u$ with a Gaussian-distributed latent factor $\bm{z}_u$ and then decodes it to $\bm{\hat{y}}_u$, a probability distribution over all items. The objective is to maximize the evidence lower bound (ELBO):
\begin{equation}
	\mathcal{L}(u)=\mathbb{E}_{\bm{z}_u \sim q_\phi(\cdot|\bm{y}_u)}[\log p_\theta(\bm{{y}}_u|\bm{z}_u)]-KL\big(q_\phi(\bm{z}_u|\bm{y}_u)||p(\bm{z}_u)\big),
	\label{eq:obj_vae-cf}
\end{equation}
where $q_\phi(\bm{z}_u|\bm{y}_u)$ is the variational posterior with parameters $\phi$ to approximate the true posterior $p(\bm{z}_u|\bm{y}_u)$. $q_\phi(\bm{z}_u|\bm{y}_u)$ is generally assumed to follow the Gaussian-distribution whose mean and variance are estimated by the encoder of Mult-VAE. That is, $\bm{z}_u\sim \mathcal{N}\big(\text{MLP}_{\mu}(\bm{y}_u;\phi), \text{diag}(\text{MLP}_{\sigma^2}(\bm{y}_u;\phi))\big)$, where $\text{MLP}_{\mu}$ and $\text{MLP}_{\sigma^2}$ denote multilayer perceptrons (MLPs). $p(\bm{z}_u)$ is the prior Gaussian distribution $\mathcal{N}(\bm{0}, \bm{I})$. $p_\theta(\bm{{y}_u}|\bm{z}_u)$ is the generative distribution conditioned on $\bm{z}_u$. The observed data $\bm{y}_u$ is assumed to be drawn from the multinomial distribution, which motivates the widely-used multinomial log-likelihood in Eq.~\eqref{eq:obj_vae-cf}:
\begin{equation*}
	\log p_\theta (\bm{{y}}_u|\bm{z}_u) = \sum_{i\in \mathcal{I}} \log p_\theta(\hat{y}_{ui}|\bm{z}_u) 
	= \sum_{i\in \mathcal{I}} \log \frac{\exp(\bm{z}_u^\top \bm{q}_i)}{\sum_{j\in \mathcal{I}} \exp(\bm{z}_u^\top \bm{q}_j)},
\end{equation*}
where $\mathcal{I}$ is the set of all items, $\bm{z}_u$ and $\bm{q}_i$ is the latent representation of user $u$ and item $i$, respectively. 

\subsection{Sampled Softmax}
Optimizing the multinomial log-likelihood of Mult-VAE is time-consuming due to the log-partition function over the logits of all items.
Given one user's inner-product logit $o_i$ for item $i$, the preference probability for the item $i$ is calculated by $P(i)=\frac{\exp(o_i)}{\sum_{j=1}^{|\mathcal{I}|}\exp(o_j)}$. Denoting model parameters by $\theta$, the gradient of the log-likelihood loss is computed as $\nabla_\theta \log P(i) = \nabla_\theta o_i-\mathbb{E}_{j\sim P}\nabla_\theta o_j$. Therefore, it takes linear costs in the number of items to compute the loss and gradient. This hinders the multinomial likelihood from the practical use in the real-world scenario with millions of items.

Sampled softmax is one popular approximation approach for log-softmax based on the self-normalized importance sampling. Since the second term of $\nabla_\theta \log P(i)$ involves an expectation, it can be approximated by sampling a small set of candidate samples $\Phi$ from a proposal $Q$. This can be equivalently achieved by adjusting $o_j'=o_j-\log Q(j), \forall j\in \{i\}\cup \Phi$ and computing the softmax over $\{i\}\cup \Phi$ (i.e. sampled softmax). Obviously, the computational cost for loss and gradient is significantly reduced. However, to guarantee the gradient of the sampled softmax unbiased, Bengio and Senécal~\cite{bengio2008adaptive} showed that the proposal $Q$ should be equivalent to the softmax distribution $P$. Since it is computationally expensive to sample from the softmax distribution, many existing methods simply use the uniform or popularity-based proposal. One recent important method~\cite{blanc2018adaptive} proposed to adopt quadratic kernel-based distributions as the proposal. However, it is not always a good approximation of the softmax distribution and suffers from a large memory footprint. Thus, it is necessary to seek a more accurate and flexible sampler.

\section{Exact Sampling with Inverted Multi-Index}\label{sec:samplers}
As demonstrated, to guarantee the gradient of the sampled softmax unbiased, it is necessary to draw candidate items from the softmax probability with the inner-product logits:
\begin{equation}
	Q(y_i|\bm{z}_u) = \frac{\exp(\bm{z}_u^\top \bm{q}_i)}{\sum_{j\in \mathcal{I}} \exp(\bm{z}_u^\top \bm{q}_j)}.
	\label{eq:proposal}
\end{equation}
To achieve this goal, inspired by the popular inverted multi-index~\cite{babenko2014inverted,JDH17,guo2020accelerating} for the approximate maximum inner product search (MIPS) and nearest neighbor search (ANNs), we provide a new way for sampling items from multiple multinomial distributions in order. Technical details will be elaborated below.

The inverted multi-index~\cite{babenko2014inverted} generalizes the inverted index with multiple codebook quantization, such as product quantization~\cite{jegou2010product} and additive quantization~\cite{babenko2014additive}. Below we demonstrate with product quantization, whose basic idea is to independently quantize multiple subvectors of indexed vectors. 
Formally, suppose $\bm{q}\in\mathbb{R}^D$ is an item vector, we first evenly split it into $m$ distinct subvectors (i.e., $\bm{q}=\bm{q}^1\oplus\bm{q}^2\oplus\cdots\oplus\bm{q}^m$ where $\oplus$ is the concatenation). Then, each subvector $\bm{q}^l$ is mapped to an element of a fixed-size vector set by a quantizer $f_l: f_l(\bm{q}^l)\in \mathcal{C}^{l}=\{\bm{c}^l_k| k\in\{1,...,K\}\}$, where $\mathcal{C}^{l}$ is the vector set (i.e. codebook) of size $K$ in the $l$-th subspace and the element $\bm{c}^l_k$ is called a codeword. Therefore, $\bm{q}$ is mapped as follows:
\begin{displaymath}
	\bm{q}\rightarrow f_1(\bm{q}^1) \oplus f_2(\bm{q}^2) \oplus \cdots \oplus f_m(\bm{q}^m)=\bm{c}^1_{k_1} \oplus \bm{c}^2_{k_2} \oplus \cdots \oplus \bm{c}^m_{k_m}.
\end{displaymath}
where $k_l (1 \le l \le m)$ is the index of the mapped codeword from $\bm{q}^l$. The codewords of each codebook can be simply determined by the K-means clustering~\cite{jegou2010product}, where the $l$-th subvectors of all items' vectors are grouped into $K$ clusters. In the following, we demonstrate sampling with 2 codebooks for simplicity (i.e. $m=2$), which is the default option of inverted multi-index. 

With the quantization, each item vector is only approximated by the concatenation of codewords. To eliminate the difference between item vector and its approximation, we add a residual vector $\bm{\tilde{q}}=\bm{q}-\bm{c}^1_{k_1} \oplus \bm{c}^2_{k_2}$ to the approximation. It is well-known that the inverted multi-index only assigns each item to a unique codeword in each subspace, making it possible to develop sublinear-time sampling methods from the softmax distribution. The following theorem lays the foundation.
\begin{theorem}\label{thm:MIDX}
	Assume $\bm{z}_u=\bm{z}_u^1 \oplus \bm{z}_u^2$ is a vector of a user $u$, $\bm{q_i}=\bm{c}^1_{k_1} \oplus \bm{c}^2_{k_2}+\bm{\tilde{q}}_i$ is a vector of an item $i$, $\Omega_{k_1,k_2}$ is the set of items which are assigned to $\bm{c}^1_{k_1}$ in the first subspace and $\bm{c}^2_{k_2}$ in the second subspace. The softmax probability $Q(y_i|\bm{z}_u)$can be decomposed as follows:
	\begin{equation}
		\small
		\begin{gathered}
			Q(y_i|\bm{z}_u)=P_u^1(k_1)\cdot P_u^2(k_2|k_1)\cdot P_u^3(y_i|k_1,k_2), \\
			P^1_u(k_1)=\frac{\psi_{k_1} \exp({\bm{z}_u^1}^\top \bm{c}^1_{k_1})}{\sum_{k=1}^K \psi_{k}\exp({\bm{z}_u^1}^\top \bm{c}^1_{k})}, \\
			P_u^2(k_2|k_1)=\frac{\omega_{k_1, k_2}\exp({\bm{z}_u^2}^\top \bm{c}^2_{k_2})}{\underbrace{\sum_{k=1}^K \omega_{k_1,k} \exp({\bm{z}_u^2}^\top  \bm{c}^2_{k})}_{\psi_{k_1}}},\quad P_u^3(y_i|k_1,k_2)=\frac{\exp(\bm{z}_u^\top \bm{\tilde{q}}_i)} {\underbrace{\sum_{j \in \Omega_{k_1, k_2}} \exp({\bm{z}_u^\top} \bm{\tilde{q}}_j)}_{\omega_{k_1,k_2}} }.
		\end{gathered}
	\label{eq:decompose}
	\end{equation}
\end{theorem}

The proof is attached in the Appendix.
Theorem~\ref{thm:MIDX} can be straightforwardly extended to the case where $m>2$. Surprisingly, this theorem provides a new perspective to exactly sample a candidate item from the softmax probability in Eq.~\eqref{eq:proposal}, which is called \textbf{MIDX} sampler. First of all, we should construct three multinomial distributions in Eq.~\eqref{eq:decompose}. Second, we sample an index $k_1$ from $P_u^1(\cdot)$, indicating to select the codeword from the first codebook $\mathcal{C}^1$. Third, we sample another index $k_2$ from the conditional probability $P_u^2(\cdot|k_1)$, indicating to select the codeword from the second codebook $\mathcal{C}^2$ given the first index. Finally, a candidate item can be sampled from the residual softmax in $P_u^3(\cdot|k_1, k_2)$. An important observation is that $\omega_{k_1, k_2}$ is absolutely not empty, such that each time an item can be sampled out in the last step. Figure~\ref{fig:sampler_procedure} illustrates the procedure and Algorithm~\ref{alg:sampling} details the workflow.

\begin{figure}[t]
	\centering
	\includegraphics[width=0.90\columnwidth]{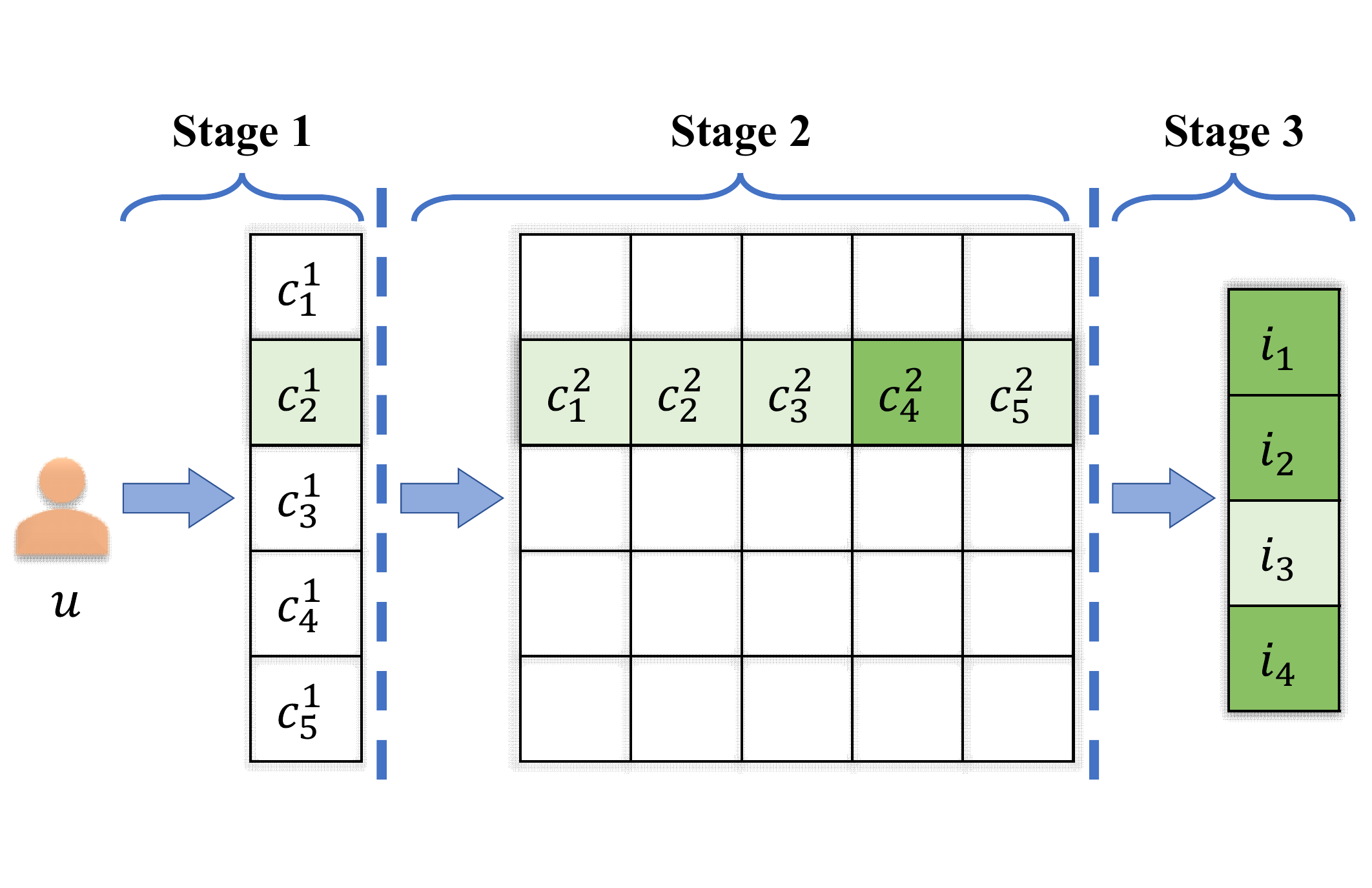}
	\vspace{-0.8cm}
	\caption{An illustration of the sampling. Firstly, draw a codeword index $(k_1=2)$ in the first codebook, and then draw another codword index $(k_2=4)$. Finally, item $i_3$ is sampled from  $\Omega_{2, 4}$, which is the set of items assigned to $k_1$ and $k_2$.}
	\label{fig:sampler_procedure}
	\vspace{-0.4cm}
\end{figure}

\noindent \textbf{Time complexity analysis.}
From Algorithm~\ref{alg:sampling}, we see that the overall procedure can be split into two parts. Lines 1-3 describe the initialization part to obtain codebooks and lines 4-13 describe the sampling part with the computation of the probability. Being independent to users, the initialization part is only executed once in $\mathcal{O}(KMDt)$, where $M$ is the number of items and $t$ is the number of iterations in K-means. Thanks to the Vose-Alias method sampling techniques~\cite{walker1977efficient}, the sampling part only takes $\mathcal{O}(1)$ time to sample an item. Unfortunately, it is necessary to compute the inner-product logits over all items, which takes $\mathcal{O}(MD)$ time.

\begin{algorithm}
	\caption{MIDX Sampling}
	\label{alg:sampling}
	\LinesNumbered
	\KwIn{Items' vectors $\{\bm{q}_i | i \in \mathcal{I}\}$, user vector $\bm{z}_u$, sampling size $T$, codebook size $K$}
	\KwOut{Candidate samples with sampling probability $(\Phi)$}
	$\mathcal{C}^{1},\mathcal{C}^{2}\leftarrow$ ProductQuantization($\{\bm{{q}}_i| i\in \mathcal{I}\}$, $K$) \;
	Compute residual vectors for all items $\{\bm{\tilde{q}}_i| i\in \mathcal{I}\}$\;
	Compute $\Omega_{k_1,k_2}, \forall 1\le k_1, k_2 \le K$ \;
	\For{$k_1=1$ \KwTo $K$}{
		\For{$k_2=1$ \KwTo $K$}{
			Compute $\omega_{k_1, k_2}$ and construct $P_u^3(\cdot|k_1,k_2)$ in Eq.~\eqref{eq:decompose}\;
		}
		Compute $\psi_{k_1}$ and construct $P_u^2(\cdot|k_1)$ in Eq.~\eqref{eq:decompose}\;
	}
	Construct $P_u^1(\cdot)$ in Eq.~\eqref{eq:decompose}\;
	Initialize $\Phi=\emptyset$\;
	\For{$i=1$ \KwTo $T$}{
		Respectively sample $k_1, k_2, i$ from $P_u^1(\cdot), P_u^2(\cdot|k_1)$ and $P_u^3(\cdot|k_1,k_2)$ in order\;
		$Q(y_i|\bm{z}_u)\leftarrow P_u^1(k_1)P_u^2(k_2|k_1)P_u^3(y_i|k_1,k_2)$\;
		$\Phi \leftarrow \Phi \cup (i, Q(y_i|\bm{z}_u))$\;
	}
	Return $\Phi$\;
\end{algorithm}

\section{Approximate Sampling with Inverted Multi-Index}
The reason why MIDX spends much time on sampling part is that it involves computing inner-product logits over all items when preparing $P_1(\cdot)$, $P_2(\cdot|k_1)$ and $P_3(\cdot|k_1,k_2)$ in Eq.~\eqref{eq:decompose}. To address this issue, we design two variants of MIDX sampling by reducing the time for computing $P_1(\cdot)$, $P_2(\cdot|k_1)$ and $P_3(\cdot|k_1,k_2)$. Although these samplers only approximate the softmax distribution, we theoretically show that the divergence between them is small.

\subsection{MIDX with Uniform}
If replacing the multinomial distribution $P_3(\cdot|k_1, k_2)$ with a non-personalized and static distribution, it will be efficient to prepare $P_1(\cdot)$ and $P_2(\cdot|k_1)$, since they only involve computing the inner product between user vector and codewords instead of the whole item vectors. A straightforward choice is the uniform distribution. The resultant variant is called \textbf{MIDX\_Uni}, whose distribution is derived based on the following theorem.
\begin{theorem}\label{thm:MIDX_uni}
	Suppose $P_1(\cdot)$ and $P_2(\cdot|k_1)$ remain the same as that in Theorem~\ref{thm:MIDX}, $P_3(\cdot|k_1,k_2)$ is replaced with a uniform distribution, i.e. $P_3(y_i|k_1,k_2)=\frac{1}{|\Omega_{k_1, k_2}|}$, where $|\Omega_{k_1, k_2}|$ denotes the number of items in the set. Then, the proposal distribution is equivalent to:
	\begin{equation}
		\begin{aligned}
			Q_{\text{uni}}(y_i|\bm{z}_u)
			& =\frac{\exp({\bm{z}_u^1}^\top \bm{c}^1_{k_1}) \exp({\bm{z}_u^2}^\top \bm{c}^2_{k_2})}{\sum_{k,k'} |\Omega_{k, k'}|\exp({\bm{z}_u^1}^\top \bm{c}^1_{k}) \exp({\bm{z}_u^2}^\top  \bm{c}^2_{k'})} \\
			& =\frac{\exp(\bm{z}_u^\top(\bm{q}_i-\bm{\tilde{q}}_i))} {\sum_{j\in \mathcal{I}}\exp(\bm{z}_u^\top(\bm{q}_j-\bm{\tilde{q}}_j))}.
		\end{aligned}
		\label{eq:midx_uni}
	\end{equation}
\end{theorem} 

The proof is attached in the Appendix. Theorem~\ref{thm:MIDX_uni} shows that each time the codeword with the large inner product and with more items are more likely to be sampled.

\noindent \textbf{Time complexity analysis.} 
When computing the sampling probability, the computation only involves the inner product between the user vector and all codewords, which takes $\mathcal{O}(KD)$ to compute. In addition, it takes $\mathcal{O}(K^2)$ since it should calculate the normalization constant in $P_2(\cdot|k_1)$ for each $k_1$. Overall, the time complexity of the preprocessing part is $\mathcal{O}(KD+K^2)$. Since the codebook size $K$ is much smaller than the number of items $M$, MIDX\_Uni sampling is much more efficient than the MIDX sampling.

\subsection{MIDX with Popularity}
Besides the uniform distribution, another widely-used static distribution is derived from popularity. If a user does not interact with a popular item, she may be truly uninterested in it since the item is highly likely to be exposed to the user. Therefore, by introducing the popularity, we design the second variant, \textbf{MIDX\_Pop}, whose distribution is derived by the following theorem.
\begin{theorem}\label{thm:MIDX_pop}
	Suppose $P_1(\cdot)$ and $P_2(\cdot|k_1)$ remain the same as that in Theorem~\ref{thm:MIDX}, $P_3(\cdot|k_1,k_2)$ is replaced with a distribution derived from the popularity, i.e. $P_3(y_i|k_1,k_2)=\frac{pop(i)}{\sum_{j \in \Omega_{k_1, k_2}} pop(j)}$, where $pop(i)$ can be any metric of the popularity. Then, the proposal distribution is equivalent to:
	\begin{equation}
		Q_{\text{pop}}(y_i|\bm{z}_u)=\frac{\exp(\bm{z}_u^\top(\bm{q}_i-\bm{\tilde{q}}_i) + \log pop(i))} {\sum_{j\in \mathcal{I}}\exp(\bm{z}_u^\top(\bm{q}_j-\bm{\tilde{q}}_j) + \log pop(j))}.
		\label{eq:midx_pop}
	\end{equation}
\end{theorem}

The proof is attached in the Appendix. Generally, let $c_i$ be occurring frequency of item $i$, $pop(i)$ can be set to $c_i$, $\log(1 +c_i)$ or $c_i^{3/4}$~\cite{mikolov2013distributed}. We empirically find that $c_i$ achieves comparatively better performance. Theorem~\ref{thm:MIDX_pop} shows that the sampling probability of an item is additionally affected by the popularity, such that the more popular items are more likely to be sampled. Regarding the time complexity, it takes $\mathcal{O}(KD+K^2)$ time in the preprocessing part,  which is the same as MIDX\_Uni.

Table~\ref{tab:time_complexity} summarizes the time and space complexity for item sampling from different proposals, which demonstrates the superiority of MIDX\_Uni and MIDX\_Pop in space and time cost. Thanks to the independence of the users, the MIDX\_Uni and MIDX\_Pop can be implemented on the GPUs, which accelerates the sampling procedure. 
Note that the initialization time refers to constructing indexes, such as alias tables, inverted multi-index or tree.

\begin{table}[t]
    \setlength{\tabcolsep}{5pt}
	\caption{Space and Time complexity of sampling $T$ items from different proposals. Denote by $M$ the number of items, $D$ the representation dimension, and $K$ the codebook size. $B$ the sample size of DNS's uniform sampling. ($K, D\ll M$)}
	\label{tab:time_complexity}
	\centering
	\vspace{-0.3cm}
	\begin{tabular}{c|c|c}
		\toprule
		Proposals $Q$ & Space &  Sample Time\\
		\midrule
		Uniform & $1$  & $T$ \\
		Popularity & $M$ & $T$  \\
		DNS~\cite{zhang2013optimizing} & $MD$ & $BDT$  \\
		Kernel~\cite{blanc2018adaptive} & $MD^2$ &  $D^2 T \log M$  \\
		MIDX in Eq.~\eqref{eq:decompose} & $MD$ &  $MD + T$ \\
		MIDX\_Uni in Eq.~\eqref{eq:midx_uni} & $KD +K^2+ M$ &  $KD + K^2 + T$ \\
		MIDX\_Pop in Eq.~\eqref{eq:midx_pop} & $KD +K^2+ M$ & $KD + K^2 + T$\\
		\bottomrule
	\end{tabular}
\end{table}

\subsection{Theoretical Analysis}
In this section, we further theoretically explain the bias of the proposed distribution from the softmax distribution.
\begin{theorem}\label{theorem:kl_midx_uni}
	Assuming that the residual embedding $ \Vert \bm{\tilde{q}}_i\Vert \leq C $, the KL divergence from the softmax distribution $ Q(\bm{y}_{\cdot}|\bm{z}_u)$ to the proposed distribution $Q_{\text{uni}}(\bm{y}_{\cdot}|\bm{z}_u)$ can be bounded from above:
	\begin{displaymath}
		0 < \mathcal{D}_{KL}\left[ Q_{\text{uni}}(\bm{y}_{\cdot}|\bm{z}_u) ||  Q(\bm{y}_{\cdot}|\bm{z}_u) \right] \le 2  C \Vert \bm{z}_u \Vert.
	\end{displaymath}
\end{theorem}

The proof is attached in the Appendix. The divergence of the proposal from Eq.~\eqref{eq:proposal} depends on $\exp(\bm{z}_u^\top\bm{\tilde{q}}_i)$. Therefore, when $\Vert \bm{\tilde{q}}_i \Vert, \forall 1\le i \le M $ (i.e., distortion of product quantization) is small, the divergence between them is small. With the increasing granularity of space partition (the number of clusters in K-means), the residual vectors are of small magnitude such that the upper bound becomes smaller. This indicates that the approximate distribution is less deviated from the softmax distribution.

\begin{theorem}
	Assuming that the residual embedding $ ||\tilde{\bm{q}}_i|| \leq C $, the KL divergence from the softmax distribution $Q(\bm{y}_{\cdot}|\bm{z}_u)$ to the proposed distribution $Q_{\text{pop}}(\bm{y}_{\cdot}|\bm{z}_u)$ can be bounded from above:
	\begin{displaymath}
		0 < \mathcal{D}_{KL}\left[ Q_{\text{pop}}(\bm{y}_{\cdot}|\bm{z}_u) ||  Q(\bm{y}_{\cdot}|\bm{z}_u) \right] \le 2  C ||\bm{z}_u|| + \log \frac{\max pop(\cdot)}{\min pop(\cdot)}.
	\end{displaymath}
\end{theorem}
The proof is attached in the Appendix.

\section{FastVAE}
We train Mult-VAE with sampled softmax, where we use the proposed proposals for item sampling (\textbf{FastVAE} for short). As shown in Section~\ref{sec:vae-cf}, the objective function in~Eq.\eqref{eq:obj_vae-cf} consists of two terms.

Regarding the first term, the expectation can be efficiently approximated by drawing a set of user vectors $\{\bm{z}_u^{(1)}, \bm{z}_u^{(2)},\cdots, \bm{z}_u^{(S)}\}$ from the variational posterior $q_\phi(\cdot|\bm{y}_u)$. By incorporating the sampled softmax, we draw a small set of candidate items $\Phi_u$ from one of our proposed samplers (i.e., MIDX\_Uni and MIDX\_Pop) and then the first term becomes:
\begin{displaymath}
	\begin{gathered}
		\mathbb{E}_{\bm{z}_u \sim q_\phi(\cdot|\bm{y}_u)}[\log p_\theta(\bm{y}_u|\bm{z}_u)]\\
		= \frac{1}{S}\sum_{s=1}^{S}\sum_{i\in \mathcal{I}}y_{ui} \log  \frac{\exp \left({\bm{z}_u^{(s)}}^\top \bm{q}_i - \log Q(y_i|\bm{z}_u^{(s)}) \right)} {\sum_{j\in \{i\} \cup \Phi_u} \exp \left({\bm{z}_u^{(s)}}^\top \bm{q}_j-\log Q(y_j|\bm{z}_u^{(s)}) \right)}.
	\end{gathered}
\end{displaymath}

For the second term, both the variational posterior $q_\phi(\bm{z}_u|\bm{y}_u)$ and the prior distribution $p(\bm{z}_u)$ follow Gaussian distributions, so that the KL divergence has a closed-form solution. Suppose $q_\phi(\bm{z}_u|\bm{y}_u)=\mathcal{N}(\bm{\mu}, \bm{\sigma}^2)$ and $p(\bm{z}_u)=\mathcal{N}(\bm{0}, \bm{1})$, the KL divergence is computed as:
\begin{displaymath}
	-KL\big(q_\phi(\bm{z}_u|\bm{y}_u)||p(\bm{z}_u)\big) = \frac{1}{2}(\log \bm{\sigma}^2-\bm{\mu}^2-\bm{\sigma}^2+\bm{1})^\top \bm{1}.
\end{displaymath}
All parameters can be jointly optimized in the objective function.

\section{Experiments}
\label{sec:experiments}
In the evaluation, the following three research questions are addressed. First, \textit{does FastVAE outperform the state-of-the-art baselines in terms of recommendation quality}? Second, \textit{how accurately do the proposal distributions approximate the softmax distribution}? Third, \textit{how efficiently are items sampled from the proposals}? More details about the experimental settings are referred to in the Appendix.  

\subsection{Experimental Settings}
\subsubsection{Datasets}
Experiments are conducted on the four public datasets for evaluation.
The \textbf{MovieLens10M}(shorted as ML10M) dataset is a classic movie rating dataset, whose ratings range from 0.5 to 5. We convert them into 0/1 indicating whether the user has rated the movie.
The \textbf{Gowalla} dataset includes users' check-ins at locations in a location-based social network and is much sparser than the MovieLens dataset.
The \textbf{Netflix} dataset is another famous movie rating dataset but with much more users.
The \textbf{Amazon} dataset is a subset of customers' ratings for Amazon books, where the rating scores are integers from 1 to 5, and books with scores higher than 4 are considered positive. For all the datasets, We filter out users and items with less than 10 interactions. The details are summerized in the Table~\ref{tab:dataset}.

\begin{table}[t]
	\caption{Dataset Statistics}
	\label{tab:dataset}
	\vspace{-0.3cm}
	\begin{tabular}{c|r|r|r|r}
		\toprule
		Dataset & \#User & \#Item & \#Interactions & Sparsity \\
		\midrule
		ML10M 	& 47,292 & 5,942 & 2,001,164 & 99.2879\%\\ 
		Gowalla & 29,858 &  40,988 & 1,027,464 & 99.9160\% \\
		Amazon 	& 56,257 & 50,154 & 1,418,076 & 99.9497\%\\
		Netflix & 422,624  & 17,618 & 53,417,358 & 99.2826\%\\
		\bottomrule
	\end{tabular}
\end{table} 

For each user, we randomly sample 80\% of interacted items to construct the history vector and fit the models to the training items. For evaluation, we take the user history to learn the necessary representations from the well-trained model and then compute metrics by looking at how well the model ranks the unseen history. 

\subsubsection{Baselines}\label{sec:baseline}
We compare our FastVAE with the following competing collaborative filtering models. The dimension of latent factors for users and items is set to 32 by default. Unless specified, we adopt the matrix factorization as the basic models.
\begin{itemize}[leftmargin=*]
  \item \textbf{WRMF}~\cite{hu2008collaborative, pan2008one}, weighted regularized matrix factorization, is a famous collaborative filtering method for implicit feedback. It sets a prior on uninteracted items associated with the confidence level of being negative. It learns parameters by alternating least square method in the case of square loss. We tune the parameter of the regularizer of uninteracted items within \{1,5,10,20,50,100,200,500\}. The coefficient of L2 regularization is fixed to 0.01.
  \item \textbf{BPR}~\cite{rendle2009bpr}, Bayesian personalized ranking for implicit feedback, utilizes the pair-wise logit ranking loss between positive and negative samples. For each pair of interacted user and item, BPR randomly samples several uninteracted items of the user for training and applies stochastic gradient descent for optimization. We set the number of sampled negative items as 5 and tune the coefficient of regularization with \{2,1,0.5\}. 
  \item \textbf{WARP-MF}~\cite{weston2010large} uses the weighted approximate-rank pair-wise loss function for collaborative filtering. Given a positive item, it uniformly samples negative items until the rating of the sampled item is higher. The rank is estimated based on the sampling trials. We use the implementation in the lightFM\footnote{https://github.com/lyst/lightfm}. The maximal number of trails is set to 50. The coefficient of the regularization is tuned within \{0.05, 0.01, 0.005, 0.001\} and the learning rate is tuned within $\{10^{-3}, 10^{-4}, 10^{-5}, 10^{-6}\}$.
  \item \textbf{AOBPR}~\cite{rendle2014improving} improves the BPR with adaptive sampling method. We use the version implemented in LibRec\footnote{https://github.com/guoguibing/librec}. The parameter for the geometric distribution is set to 500 and the learning rate is set to 0.05. We tune the coefficient of the regularization within \{0.005, 0.01, 0.02\}.
  \item \textbf{DNS}~\cite{zhang2013optimizing} dynamically chooses items according to the predicted ranking list for the topk recommendation. Specifically, the dynamic sampler first draws samples uniformly from the item set and the item with the maximum rating is selected. 
  \item \textbf{PRIS}~\cite{lian2020personalized}    utilizes the importance sampling to the pairwise ranking loss for personalized ranking and assigns the sampling weight to the sampled items. We adopt the joint model implemented in the open resource code\footnote{https://github.com/DefuLian/PRIS}. The number of clusters is set to 16.
  \item  \textbf{Self-Adversarial (SA)}~\cite{sun2019rotate}, a self-supervised method for negative sampling, is recently proposed for the recommendation. It utilizes uniform sampling and assigns the sampling weight for the negative item depending on the current model.
  \item \textbf{Mult-VAE}~\cite{liang2018variational}, variational autoencoders for collaborative filtering, is the work of learning user representations with variational autoencoders in recommendation systems. It learns the user representation by aiming at maximizing the likelihood of user click history. We mainly focus on the comparison with VAE and we will introduce the parameter setting in the following part. 
\end{itemize}
In addition to these recommendation algorithms, we conduct experiments with the following samplers.
\begin{itemize}[leftmargin=*]
  \item \textbf{Uniform} sampler is a common sampling strategy that randomly draws negatives from the set of items for optimization, widely used for sampled softmax. 
  \item \textbf{Popularity} sampler is correlated with the popularity of items, where the items with higher popularity have a greater probability of being sampled. The popularity is computed as $\log(f_i+1)$ where $f_i$ is the occurring frequency of item $i$. We normalize the popularity of all items for sampling. 
  \item \textbf{Kernel} based sampler~\cite{blanc2018adaptive} is a recent method for adaptively sampled softmax, which lowers the bias by the non-negative quadratic kernel. Furthermore, the kernel-based sampler is implemented with divide and conquer depending on the tree structure.
\end{itemize}

\subsubsection{Evaluation Metrics}
Two standard metrics are utilized for evaluating the quality of recommendation, Normalized Discounted Cumulative Gain (NDCG) and Recall. A higher NDCG@$k$ represents the positive items in the test data are ranked higher in the ranking list. Recall@$k$ measures the fraction of the positive items in the test data. All algorithms are fine-tuned based on NDCG@50. After that, we run 5 times cross-validation. 

\subsubsection{Experiment Settings}
We develop the proposed algorithms FastVAE with Pytorch in a Linux system (2.10 GHz Intel Xeon Gold 6230 CPUs and a Tesla V100 GPU). We utilize the Adam algorithm with a weight decay of 0.01 for optimization. We implement the variational autoencoder with one hidden layer and the generative module would be $[M \rightarrow 200 \rightarrow 32]$. The active function between layers is ReLu by default. The input of user history is dropout with a probability of 0.5 before the linear layers. The batch size is set to 256 forall the datasets. The learning rate is tuned over \{0.1,0.01,0.001,0.0001\}. We train the models within 200 epochs.

For the FastVAE, the number of samples is set to 200 for the MovieLens10M and Netflix dataset, 1000 for the Gowalla dataset, 2000 for the Amazon dataset. There are 16 codewodes for each code book. Regarding the popularity based strategy, we follow the same popularity function $\log(f_i + 1)$.

\subsection{Comparisons with Baselines}
The comparisons of recommendation quality (i.e., Recall@50 and NDCG@50) with baselines is reported in Table~\ref{tab:comparisons_baselines}, which are based on $5$-time independent trials. We report the results of FastVAE with MIDX\_Pop here. We have the following findings.

\begin{table*}[ht]
	\centering
	\setlength{\tabcolsep}{3pt}
	\caption{Comparisons with baselines w.r.t NDCG@50 and Recall@50 ($\Delta=10^{-4}$).}
	\label{tab:comparisons_baselines}
	\vspace{-0.3cm}
	\begin{tabular}{c|c|c|c|c|c|c|c|c}
		\toprule
		& \multicolumn{2}{c|}{MovieLens-10M}&\multicolumn{2}{c|}{Gowalla} & \multicolumn{2}{c|}{Netflix} & \multicolumn{2}{c}{Amazon}  \\ 
		\cmidrule(){2-9}
		& NDCG@50 & Recall@50 & NDCG@50 & Recall@50 & NDCG@50 & Recall@50 & NDCG@50 & Recall@50 \\ 
		\midrule
		WRMF  & 0.3194$\pm$0.3$\Delta$ & 0.4967$\pm$0.6$\Delta$ & 0.1316$\pm$0.1$\Delta$ & 0.2223$\pm$0.1$\Delta$ & 0.3020$\pm$0.1$\Delta$ & 0.3653$\pm$0.6$\Delta$ & 0.0919$\pm$1.4$\Delta$ & 0.1802$\pm$3.0$\Delta$\\
		BPR   & 0.2915$\pm$2.4$\Delta$ & 0.4642$\pm$3.2$\Delta$& 0.1216$\pm$1.1$\Delta$ & 0.1978$\pm$1.9$\Delta$ & 0.2742$\pm$1.7$\Delta$ & 0.3283$\pm$1.2$\Delta$ & 0.0740$\pm$2.2$\Delta$ & 0.1441$\pm$4.2$\Delta$\\
		WARP-MF  & 0.2968$\pm$2.3$\Delta$ & 0.4785$\pm$3.3$\Delta$ & 0.1273$\pm$0.7$\Delta$ & 0.2073$\pm$1.7$\Delta$ & 0.2953$\pm$1.2$\Delta$ & 0.3539$\pm$1.2$\Delta$ & 0.0798$\pm$1.4$\Delta$ & 0.1615$\pm$3.7$\Delta$\\
		AOBPR & 0.2934$\pm$0.5$\Delta$ & 0.4753$\pm$0.3$\Delta$ & 0.1385$\pm$0.4$\Delta$ & 0.2369$\pm$0.8$\Delta$ & 0.2952$\pm$0.4$\Delta$ & 0.3560$\pm$0.8$\Delta$ & 0.0906$\pm$1.7$\Delta$ & 0.1763$\pm$2.5$\Delta$\\
		DNS & 0.3153$\pm$2.7$\Delta$ & 0.4988$\pm$3.4$\Delta$ & 0.1622$\pm$1.5$\Delta$ & 0.2761$\pm$2.9$\Delta$ & 0.2974$\pm$2.4$\Delta$ & 0.3594$\pm$2.5$\Delta$ & 0.1119$\pm$1.7$\Delta$ & 0.2186$\pm$3.4$\Delta$ \\
		PRIS & 0.3162$\pm$1.6$\Delta$& 0.4937$\pm$3.3$\Delta$ &0.1657$\pm$2.7$\Delta$ & 0.2736$\pm$3.1$\Delta$ & 0.2975$\pm$2.5$\Delta$ & 0.3608$\pm$2.9$\Delta$ & 0.1189$\pm$2.9$\Delta$ & 0.2244$\pm$4.1$\Delta$\\
		SA & 0.3237$\pm$1.4$\Delta$ & 0.5066$\pm$2.0$\Delta$ & 0.1704$\pm$1.2$\Delta$ & 0.2866$\pm$1.8$\Delta$ & 0.3177$\pm$3.5$\Delta$ & 0.3784$\pm$2.8$\Delta$ & 0.1378$\pm$1.8$\Delta$ & 0.2401$\pm$2.9$\Delta$ \\
		Mult-VAE& 0.3206$\pm$2.5$\Delta$ & 0.5037$\pm$2.7$\Delta$ & 0.1751$\pm$3.2$\Delta$ & 0.2911$\pm$5.7$\Delta$ & 0.3227$\pm$2.8$\Delta$ & 0.3841$\pm$3.1$\Delta$ & 0.1441$\pm$0.9$\Delta$ & 0.2483$\pm$2.9$\Delta$\\
		\textbf{FastVAE} & \textbf{0.3275}$\pm$2.5$\Delta$ & \textbf{0.5078}$\pm$2.4$\Delta$ & \textbf{0.1797}$\pm$2.0$\Delta$ & \textbf{0.2971}$\pm$2.1$\Delta$ & \textbf{0.3238}$\pm$3.0$\Delta$ & \textbf{0.3845}$\pm$2.7$\Delta$ & \textbf{0.1404}$\pm$2.1$\Delta$ & \textbf{0.2434}$\pm$3.5$\Delta$ \\ 
		\bottomrule
	\end{tabular}
\end{table*}

\textit{Finding 1: By using the MIDX-like proposals, FastVAE with sampled softmax could behaves almost as well as Multi-VAE with full softmax and even perform slightly better.} Surprisingly, the averaged relative improvements are even up to 0.64\% and 0.25\% on four datasets in terms of NDCG@50 and Recall@50, respectively. This implies the MIDX-like proposals could accurately approximate the softmax distribution and sample informative items. The improvements may lie in the oversampling of less popular items.

\textit{Finding 2: FastVAE outperforms all state-of-the-art baselines on two datasets.} The averaged relative improvements over the best baseline are up to 2.61\% and 1.72\% in terms of NDCG@50 and Recall@50, respectively. This indirectly implies the effectiveness of the proposed samplers at sampling high-informative items. Note that WRMF usually works better than static-sampling-based baselines, as WRMF treats all unobserved data as negative. However, the lack of differentiation among them leads to sub-optimal solutions compared to Mult-VAE, whose objective function (i.e. full softmax) also takes all items into account.

\subsection{Comparisons with Different Samplers}
\subsubsection{Divergence between Proposals and the Softmax Distribution}
In order to understand how accurately the proposal distributions approximate the softmax distribution, we investigate the divergence between the proposals and the softmax distribution on the MovieLens10M dataset. In particular, we randomly select a user, and compute her/his softmax distribution with a randomly-initialized model and well-trained model, respectively. Regarding the proposals, we sample 100,000 items from each of them and then plot the cumulative probability distribution. Regarding MIDX-like proposals in both cases, 64 clusters in each quantization will be used. The results of these two cases are reported in Figure~\ref{fig:distribution_start_end}, where items are sorted by popularity for better comparison. Note that since we have observed similar results from multiple users, only one user's result is reported for illustration. We have the following findings.

\begin{figure}[t]
	\centering
	\subfigure[Randomly-initialized model]{\includegraphics[width=0.49\columnwidth]{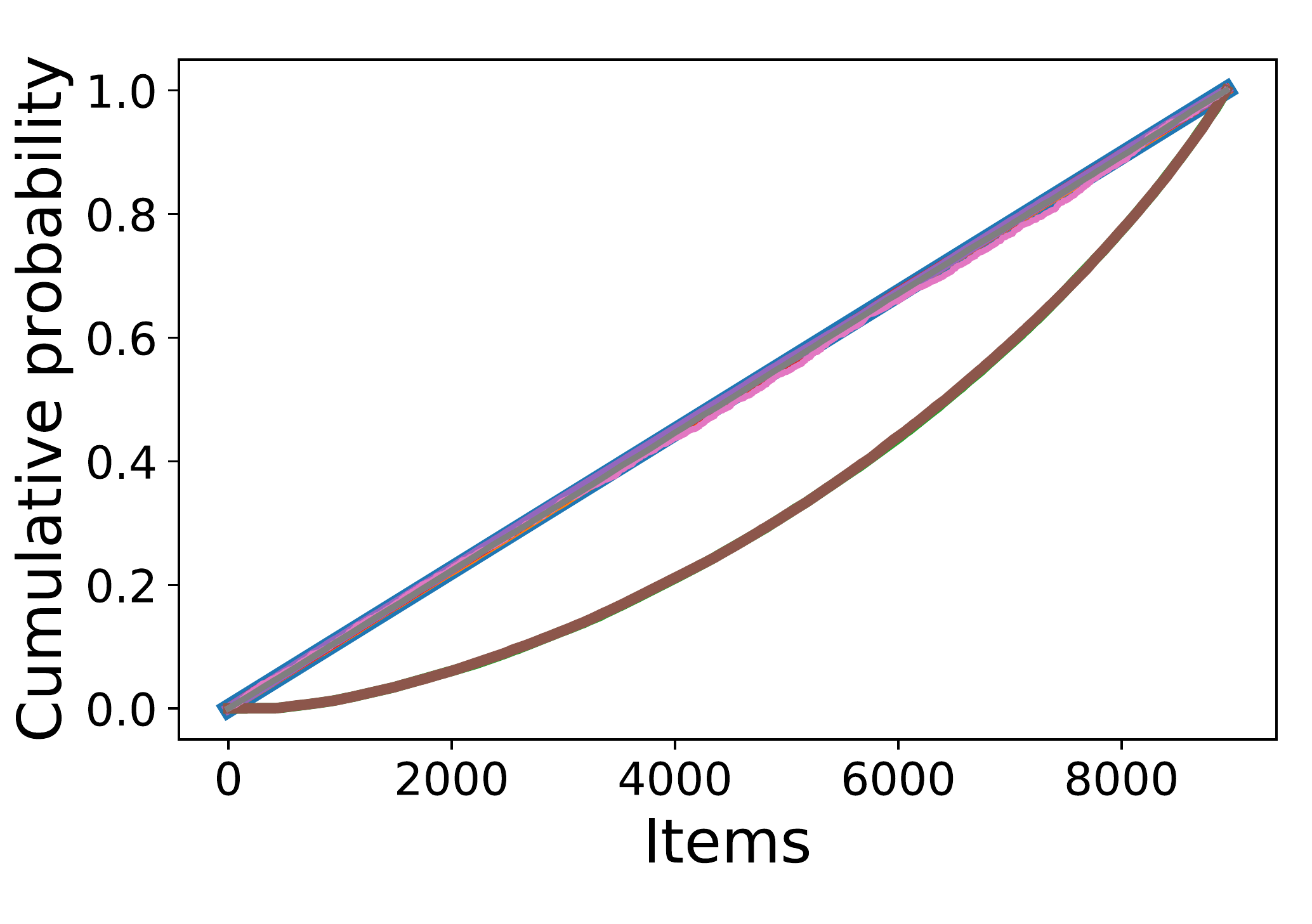}}
	\subfigure[Well-trained model]{\includegraphics[width=0.49\columnwidth]{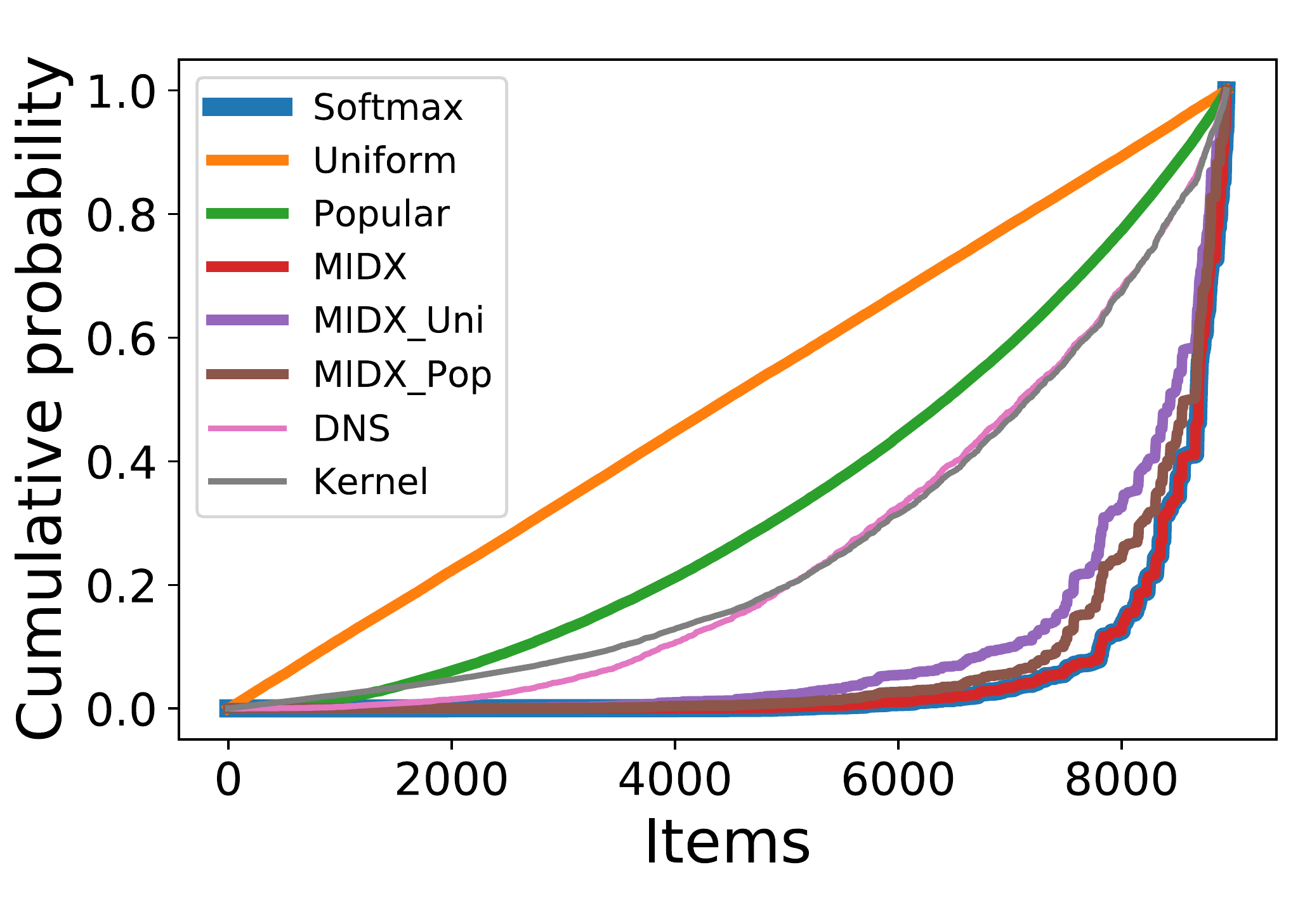}}
	\vspace{-0.3cm}
	\caption{Cumulative probability distribution of different samplers. Items are sorted by popularity.} 
	\label{fig:distribution_start_end}
\end{figure}

\textit{Finding 1: The MIDX sampler is as accurate as softmax, based on full coincidence between softmax and MIDX in both cases.} This is because the decomposition of the softmax distribution is fully exact, as shown in Section~\ref{sec:samplers}. However, since the item should be sampled from the residual softmax distribution, the time cost is so high that the MIDX sampler is not directly used in practice.

\textit{Finding 2: When the model is well-trained, the MIDX-variant samplers are much closer to the softmax distribution than Kernel and DNS.} This implies that MIDX-variant samplers reduce the bias of sampled softmax. Though Kernel-based sampling also directly approximates the softmax distribution, it is almost as close as DNS to the softmax. Moreover, the Kernel-based sampling approximates the probability of long-tailed items less accurately. This is consistent with the fact the Kernel-based sampler oversamples items with negative logits \cite{blanc2018adaptive} since long-tailed items are more likely to yield negative logits. 

\textit{Finding 3: The MIDX\_Uni sampler approximates the softmax distribution a little less accurately than MIDX\_Pop.} 
These samplers mainly vary in item sampling in the last stage, but all depend on item vector quantization. This implies the effect of inverted multi-index at approximate sampling. Moreover, compared to the softmax distribution, these samplers are more likely to sample less popular items, evidenced by that their curves are slightly above the softmax.

\textit{Finding 4: The MIDX-variant samplers can capture the dynamic update of the model.} In particular, when the model is well-trained, the MIDX-variant samplers well approximate the softmax; when the model is only randomly initialized, most dynamic samplers are similar to the static samplers. The latter observation is reasonable since randomized representations do not have cluster structures. This indicates that along with the training course of the model, the MIDX-variant samplers can be more and more informative.

\subsubsection{Effectiveness Study of Samplers}
To validate the effectiveness of the proposed MIDX-based samplers, we investigate the recommendation performance during training epochs with different samplers aforementioned in the section \ref{sec:baseline}. We report the changing curve of NDCG@50 and Recall@50 on the Gowalla and Netflix datasets in Figure~\ref{fig:samplers_effect}. We sample 1,000 items for the Gowalla dataset and 200 items for the Netflix dataset with all the tested samplers. The number of the sampled items are greatly smaller than the number of total items. 
\begin{figure}
	\subfigure[Gowalla]{
		\includegraphics[width=0.485\columnwidth]{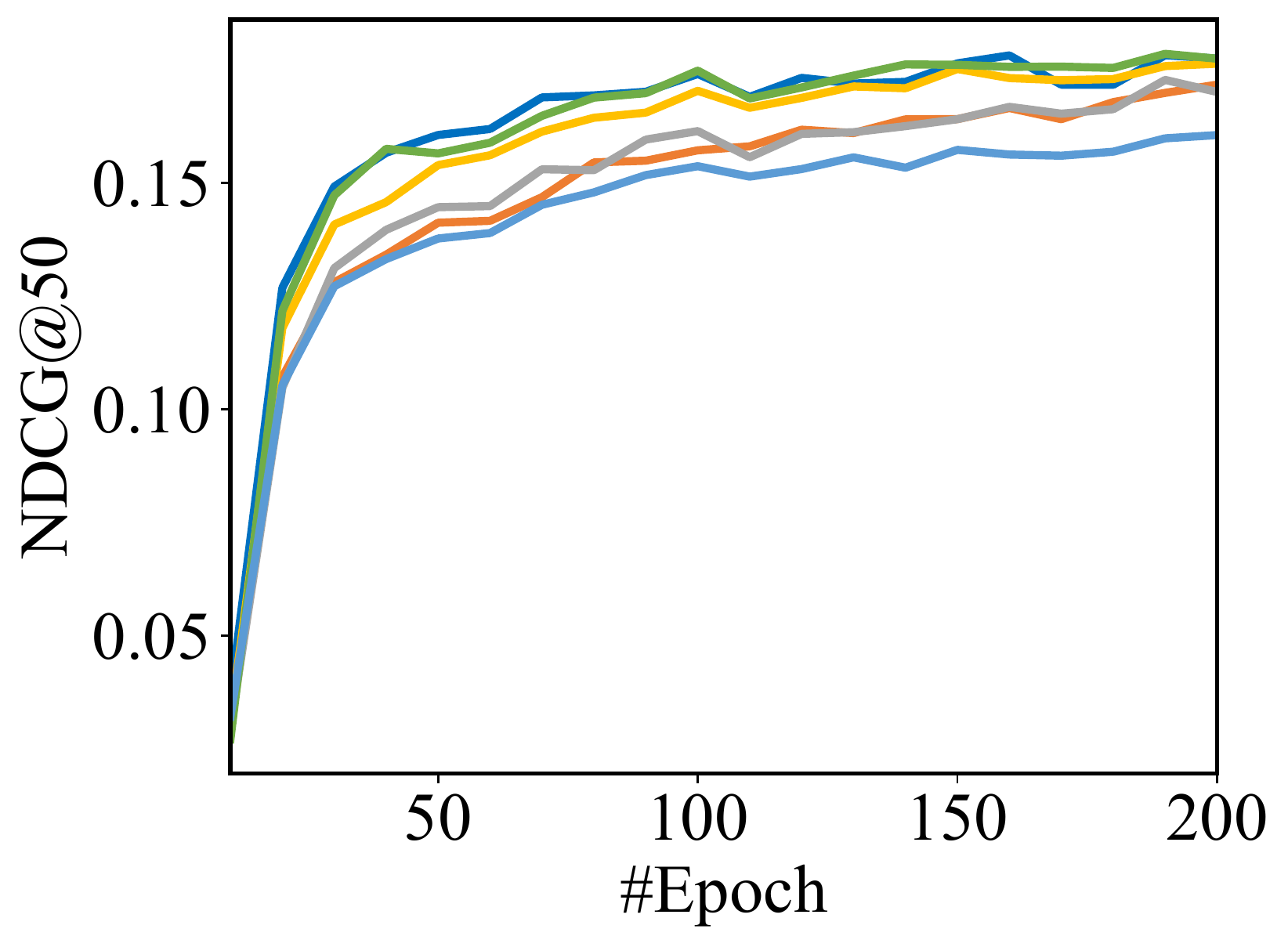}}
	\subfigure[Netflix]{
		\includegraphics[width=0.485\columnwidth]{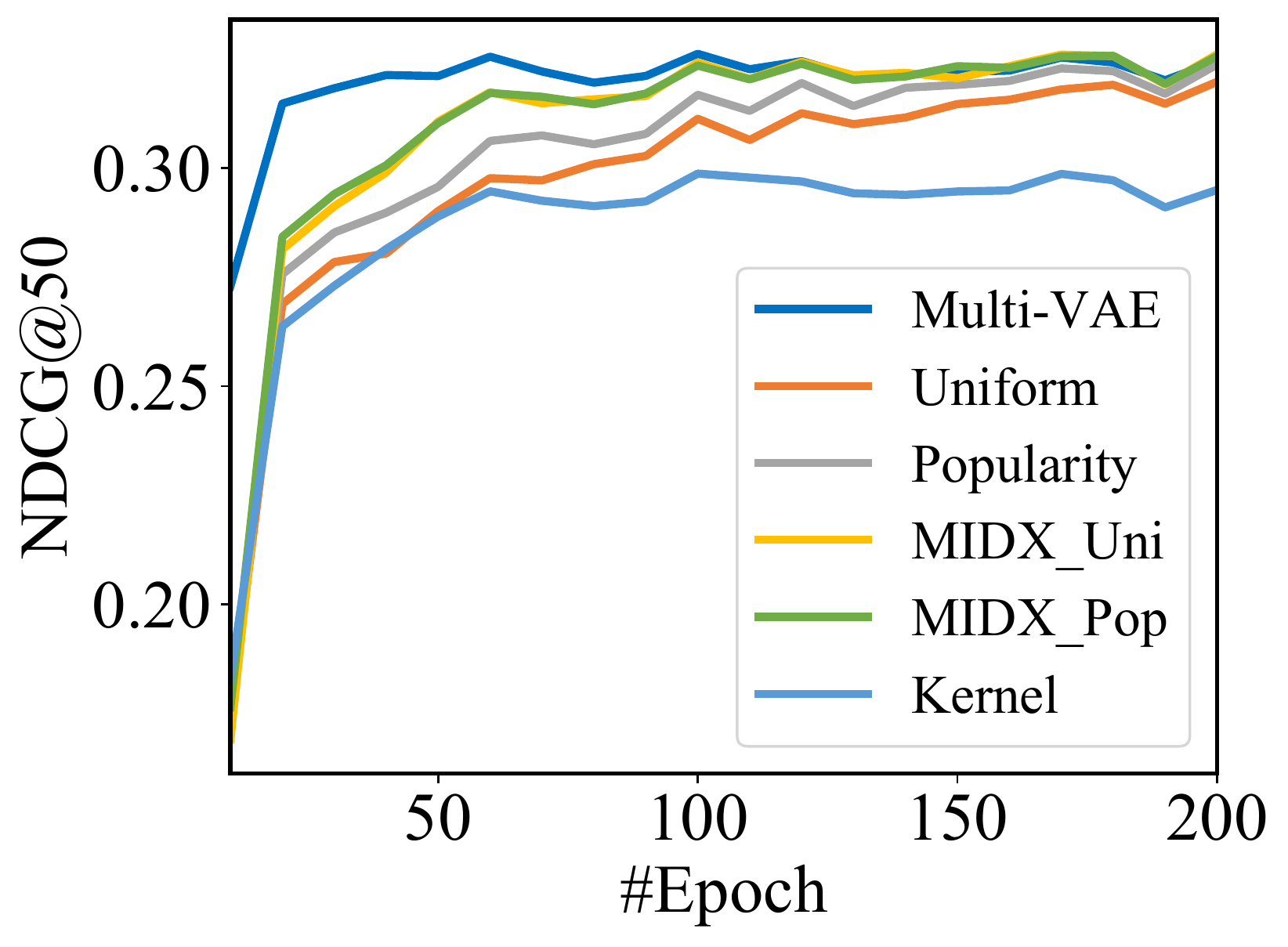}}
	\vspace{-0.5cm}
	\caption{Effectiveness of different samplers in terms of recommendation performance.}
	\label{fig:samplers_effect}
\end{figure}

From the figure, we have the following finding: \textit{The MIDX-based samplers contributes to faster convergence compared to the baseline samplers.} Compared with the static samplers, i.e. Uniform and Popularity, the MIDX-based samplers tends to sample more informative items so that they defeat the samplers during the whole training process. Although the Kernel sampler has a better estimation of the softmax distribution and can capture the dynamics of the softmax distribution, the sampling probability is not well attached for calculating the sampled softmax loss, so that it perform bad in terms of the recommendation quality. On the more sparse dataset, Gowalla, the MIDX\_Uni and MIDX\_Pop perform as well as the Multi-VAE and even has slightly better performance. This may implies the oversampling of the full softmax.

\subsubsection{Efficiency Study of Samplers}
Though MIDX-variant samplers could produce a good approximation to the softmax, it is still unclear how efficiently items are sampled. Therefore, we increase the number of items in the Amazon dataset while keeping the sample size at 200. Experiments are run for 5 times and we report the average running time of each epoch w.r.t the sampling and training time in Figure~\ref{fig:samplers_time}. The \textit{training Time} contains the sampling time and inference time since the sampling procedure is implemented after the user encoding. We compare the running time with the Multi-VAE and report the speedup. The running time of the Kernel sampler is not reported here because it is difficult to implement in GPUs and is substantially longer than the other samplers.

\begin{figure}[t]
	\centering
	\subfigure[Sampling Time]{
		\includegraphics[width=0.475\columnwidth]{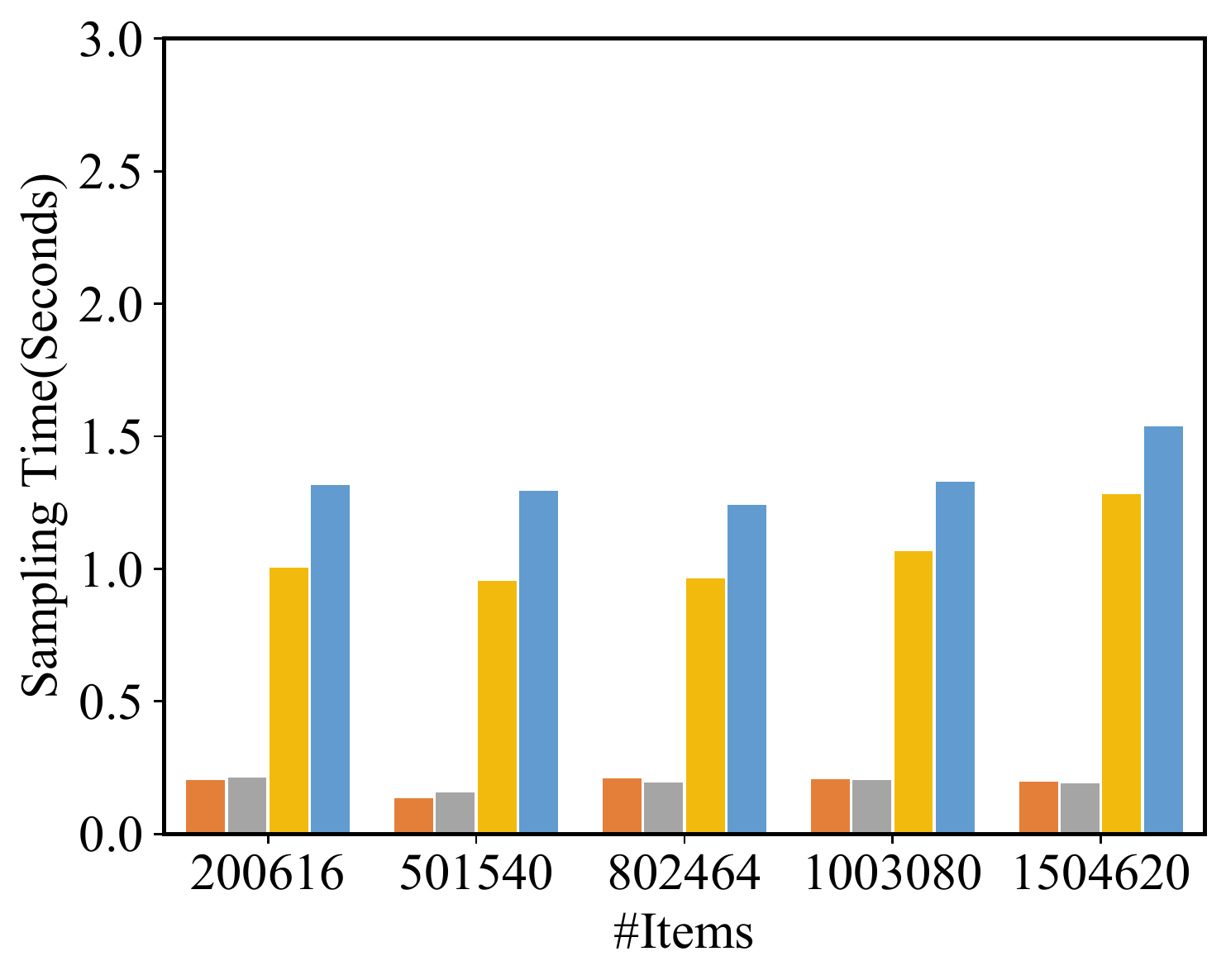}
	}
	\subfigure[Training Speedup]{
		\includegraphics[width=0.47\columnwidth]{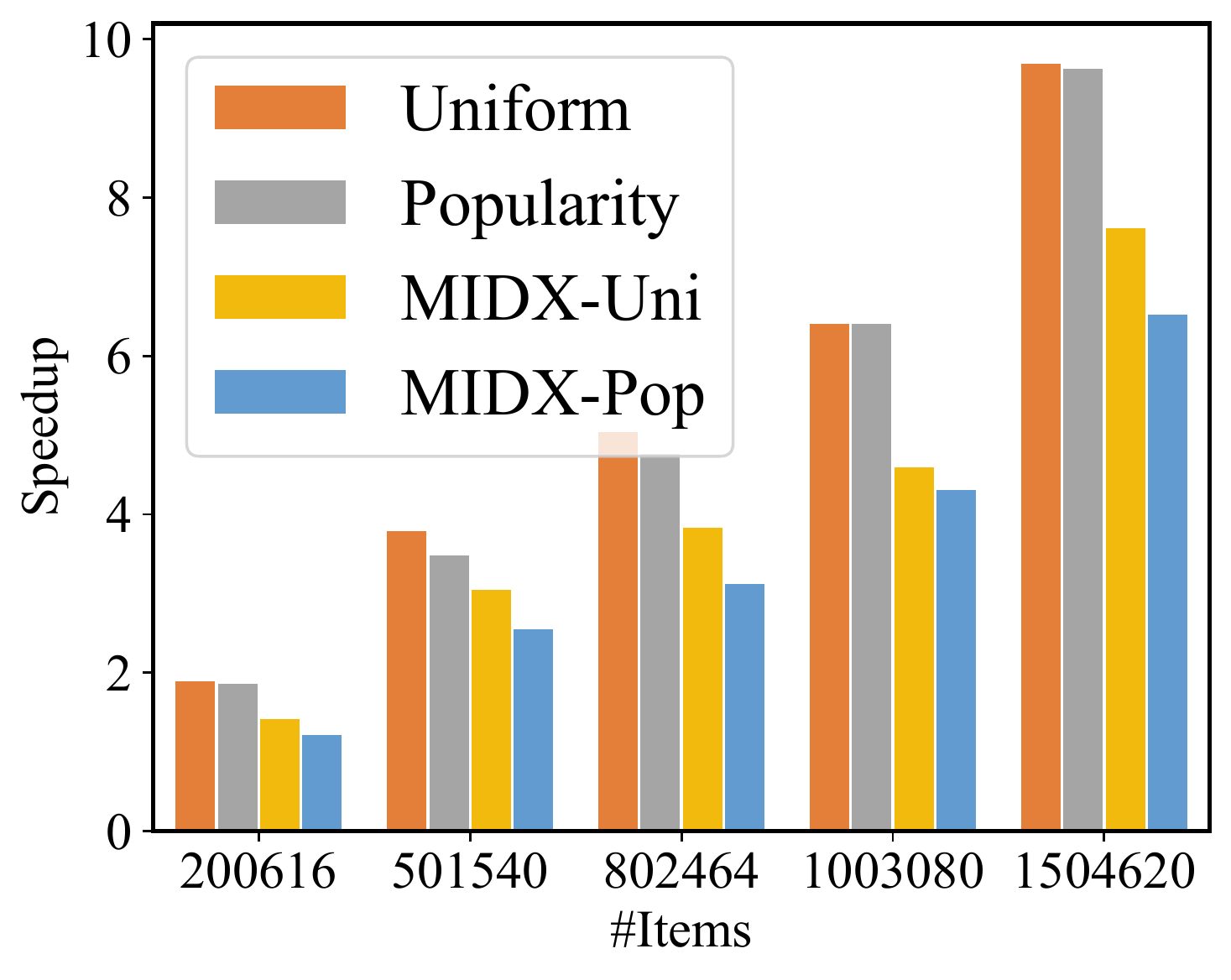}
	}
	\vspace{-0.5cm}
	\caption{Average Running Time v.s. Item numbers.}	\label{fig:samplers_time}
\end{figure}

From this figure, we observe that \textit{The MIDX\_Uni is efficient than MIDX\_Pop sampler, but less efficient than static samplers.} The static samplers require less than 0.5 seconds to sample items during each epoch, while the MIDX-based samplers take about 1.5 seconds. Indeed, as the number of items increases, the training time of Multi-VAE increases from 2.3(s) to 10.2(s), which is substantially longer than the sampling time. With the increasing of items, the training process is substantially more accelerated. The MIDX-based sampler also accelerates the training time more than six times when the number of items reaches about 1.5 million, confirming the suggested samplers' great efficiency.

\subsection{Sensitivity Analysis}
\subsubsection{Number of negative samples.} We conduct experiments on the Gowalla datasets with the \textit{Uniform} and \textit{MIDX\_Uni} sampler, as shown in Figure~\ref{fig:diff_negs}. The numbers of negative items are varied in \{200, 1000,5000\}.  
\textit{Our proposed MIDX based samplers show superior performances even if the number of sampled items is modest.} With the increasing of the sample numbers, the two samplers perform better in terms of NDCG@50. When the number of negative items is greatly small, the MIDX\_Uni also improves dramatically, indicating the good estimation of the softmax.

\begin{figure}[t]
	\centering
	\subfigure[Number of Negative Samples]{\includegraphics[width=0.485\columnwidth]{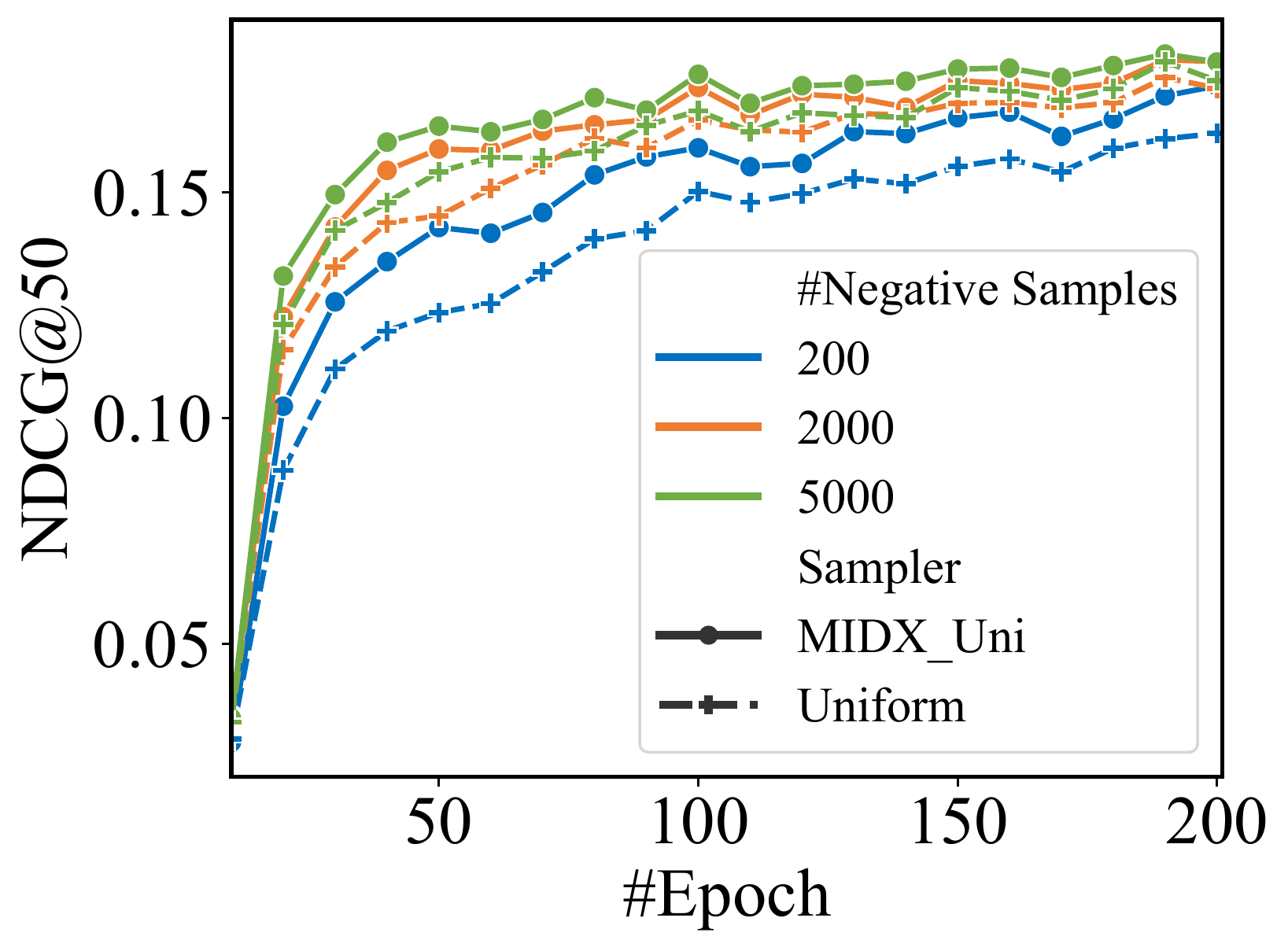}
	\label{fig:diff_negs}}
	\subfigure[Number of Clusters]{\includegraphics[width=0.485\columnwidth]{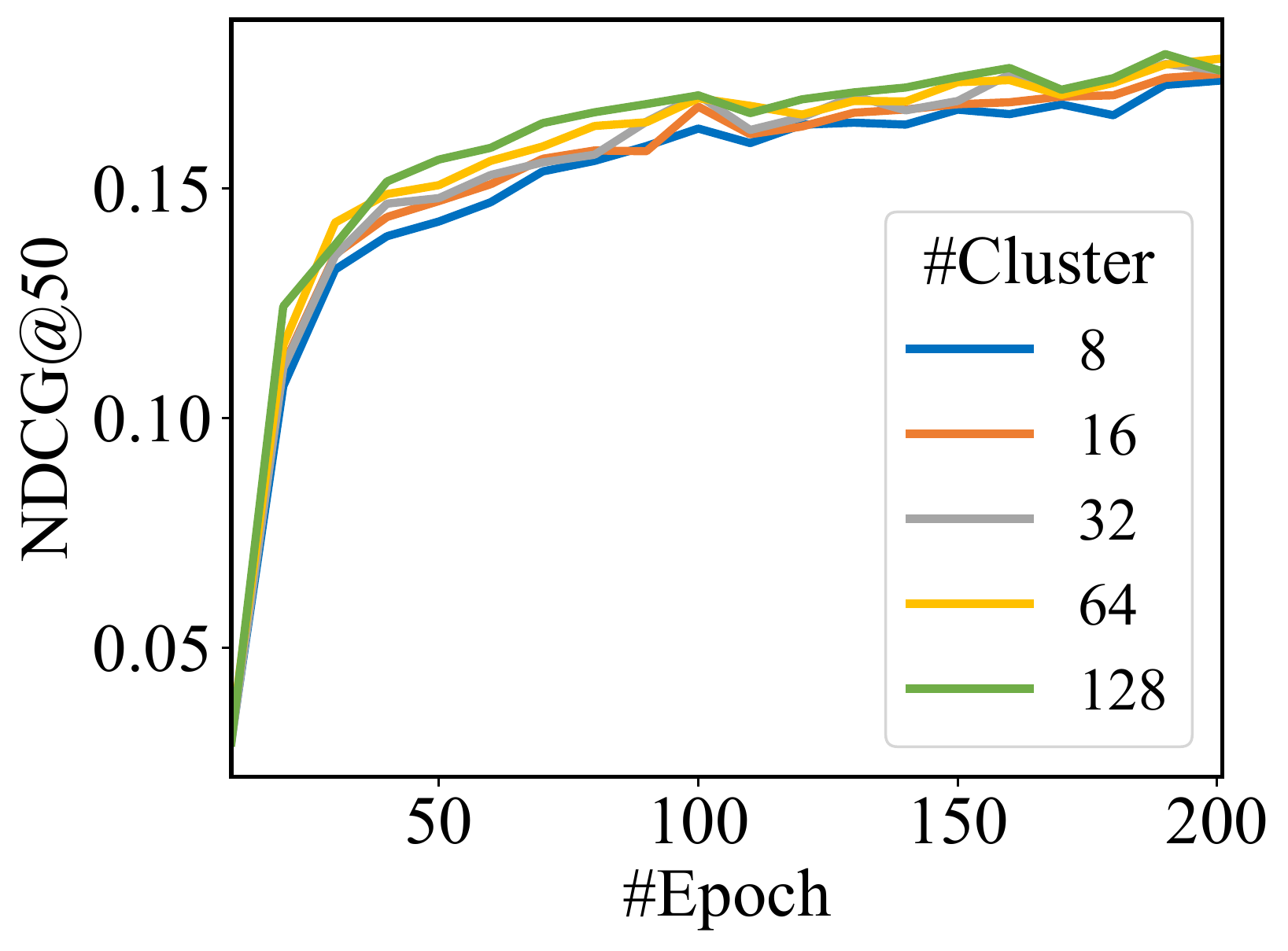}
	\label{fig:diff_cluster}}
	\vspace{-0.5cm}
	\caption{Sensitive analysis on the Gowalla dataset.}
\end{figure}

\subsubsection{Number of clusters.}
The number of clusters can greatly influence the performance of the approximation, as analysed in the Theorem~\ref{theorem:kl_midx_uni}. We further validate the influence of the cluster numbers in terms of the recommendation quality. The numbers of clusters are varied in $\{8,16,32,64,128\}$. We report the running curve in Figure~\ref{fig:diff_cluster}. With the increasing of the cluster number, the MIDX\_Uni sampler performs better in the initial training epochs, indicating the better estimation with more clusters. Meanwhile, the MIDX\_Uni also behave well with less clusters, demonstrating the robustness of the MIDX\_Uni sampler with respect to the cluster number.

\section{Conclusion} 
In this paper, we discover the high-quality approximation of the softmax distribution by decomposing the softmax probability with the inverted multi-index, and design efficient sampling procedures, from which items can be independently sampled in sublinear or even constant time. These approximate samplers are exploited for fast training the variational autoencoder for collaborative filtering. The experiments on the three public real-world datasets demonstrate that the FastVAE outperforms the state-of-the-art baselines in terms of sampling quality and efficiency. 

\begin{acks}
  The work is supported by the National Natural Science Foundation of China (No. 62022077, 61972069, 61836007 and 61832017), and Shenzhen Municipal Science and Technology R\&D Funding Basic Research Program (JCYJ20210324133607021).
\end{acks}

\bibliographystyle{ACM-Reference-Format}
\bibliography{myref}

\appendix
\section{Appendix}

In the appendix, we provide the proofs of theorem 5.1, theorem 5.2,theorem 5.3 and theorem 5.4.

For better illustration, we review some important notations here. In the following, denote by $\mathcal{I}$ the set of items, $\bm{z}_u$ a vector of the user $u$, $\bm{q}_i$ a vector of the item $i$. The softmax probability with the inner-product logits can be compudated by:
\begin{displaymath}
	Q(y_i|\bm{z}_u) = \frac{\exp(\bm{z}_u^\top \bm{q}_i)}{\sum_{j\in \mathcal{I}} \exp(\bm{z}_u^\top \bm{q}_j)}.
\end{displaymath}
Particularly, $\bm{q}_i$ can be decomposed based on the codebooks. It is formulated as $\bm{q}_i=\bm{c}^1_{k_1} \oplus \bm{c}^2_{k_2}+\bm{\tilde{q}}_i$ where $\bm{c}^i_{k_i} $ is the $k_i$-th codeword index of $i$-th codebook and $\bm{\tilde{q}}_i$ is the residual vector.

\begin{theorem}[\textbf{Theorem 4.1}]
	Assume $\bm{z}_u=\bm{z}_u^1 \oplus \bm{z}_u^2$ is a vector of a user $u$, $\bm{q_i}=\bm{c}^1_{k_1} \oplus \bm{c}^2_{k_2}+\bm{\tilde{q}}_i$ is a vector of an item $i$, $\Omega_{k_1,k_2}$ is the set of items which are assigned to $\bm{c}^1_{k_1}$ in the first subspace and $\bm{c}^2_{k_2}$ in the second subspace. The softmax probability $Q(y_i|\bm{z}_u)$ can be decomposed as follows:
	\begin{equation}
		\begin{gathered}
			Q(y_i|\bm{z}_u)=P_u^1(k_1)\cdot P_u^2(k_2|k_1)\cdot P_u^3(y_i|k_1,k_2), \\
			P^1_u(k_1)=\frac{\psi_{k_1} \exp({\bm{z}_u^1}^\top \bm{c}^1_{k_1})}{\sum_{k=1}^K \psi_{k}\exp({\bm{z}_u^1}^\top \bm{c}^1_{k})}, \\
			P_u^2(k_2|k_1)=\frac{\omega_{k_1, k_2}\exp({\bm{z}_u^2}^\top \bm{c}^2_{k_2})}{\underbrace{\sum_{k=1}^K \omega_{k_1,k} \exp({\bm{z}_u^2}^\top  \bm{c}^2_{k})}_{\psi_{k_1}}},\\
			P_u^3(y_i|k_1,k_2)=\frac{\exp(\bm{z}_u^\top \bm{\tilde{q}}_i)} {\underbrace{\sum_{j \in \Omega_{k_1, k_2}} \exp({\bm{z}_u^\top} \bm{\tilde{q}}_j)}_{\omega_{k_1,k_2}} }.
		\end{gathered}
	\end{equation}
\end{theorem}

\begin{proof}
	\begin{displaymath}
		\small
		\begin{split}
			Q(y_i|\bm{z}_u) =& \frac{\exp(\bm{z}_u^\top \bm{q}_i)}{\sum_{j\in \mathcal{I}} \exp(\bm{z}_u^\top \bm{q}_j)}\\
			= & \frac{\exp({\bm{z}_u^1}^\top \bm{c}^1_{k_1}) \exp({\bm{z}_u^2}^\top \bm{c}^2_{k_2}) \exp(\bm{z}_u^\top \bm{\tilde{q}}_i)}
			{\sum_{k=1}^K \exp( {\bm{z}_u^1}^\top  \bm{c}^1_k )  \underbrace{\sum_{k'=1}^K \exp( {\bm{z}_u^2}^\top \bm{c}^2_{k'} ) \sum_{j \in \Omega_{k,k'}} \exp({\bm{z}_u^\top} \bm{\tilde{q}}_{j})}_{\Psi_k}} \\
			=& \frac{\Psi_{k_1}\exp({\bm{z}_u^1}^\top \bm{c}^1_{k_1}) }{\sum_{k=1}^K \Psi_k\exp( {\bm{z}_u^1}^\top  \bm{c}^1_k )} \cdot \frac{\exp({\bm{z}_u^2}^\top \bm{c}^2_{k_2}) \exp(\bm{z}_u^\top \bm{\tilde{q}}_i)}{\Psi_{k_1}} \\
			=& P_u^1(k_1)\cdot \frac{\exp({\bm{z}_u^2}^\top \bm{c}^2_{k_2}) \exp(\bm{z}_u^\top \bm{\tilde{q}}_i)}
			{\sum_{k=1}^K \exp( {\bm{z}_u^2}^\top \bm{c}^2_{k} ) \underbrace{\sum_{j \in \Omega_{k_1,k}} \exp({\bm{z}_u^\top} \bm{\tilde{q}}_{j})}_{\omega_{k_1, k}}}\\
			=&P_u^1(k_1)\cdot \frac{\omega_{k_1, k_2}\exp({\bm{z}_u^2}^\top \bm{c}^2_{k_2})}{\sum_{k=1}^K \omega_{k_1, k}\exp( {\bm{z}_u^2}^\top \bm{c}^2_{k} )} \cdot
			\frac{\exp(\bm{z}_u^\top \bm{\tilde{q}}_i)}{\omega_{k_1, k_2}}\\
			=&P_u^1(k_1)\cdot P_u^2(k_2|k_1)\cdot P_u^3(y_i|k_1,k_2).
		\end{split}
	\end{displaymath}
\end{proof}

\begin{theorem}[\textbf{Theorem 5.1}]
	Suppose $P_1(\cdot)$ and $P_2(\cdot|k_1)$ remain the same as that in Theorem 4.1, $P_3(\cdot|k_1,k_2)$ is replaced with a uniform distribution, i.e. $P_3(y_i|k_1,k_2)=\frac{1}{|\Omega_{k_1, k_2}|}$ where $|\Omega_{k_1, k_2}|$ denotes the number of items in the set. Then, the proposal distribution is equivalent to:
	\begin{displaymath}
		\begin{split}
		Q_{\text{uni}}(y_i|\bm{z}_u)
		& =\frac{\exp({\bm{z}_u^1}^\top \bm{c}^1_{k_1}) \exp({\bm{z}_u^2}^\top \bm{c}^2_{k_2})}{\sum_{k,k'} |\Omega_{k, k'}|\exp({\bm{z}_u^1}^\top \bm{c}^1_{k}) \exp({\bm{z}_u^2}^\top  \bm{c}^2_{k'})}\\
		& =\frac{\exp(\bm{z}_u^\top(\bm{q}_i-\bm{\tilde{q}}_i))} {\sum_{j\in \mathcal{I}}\exp(\bm{z}_u^\top(\bm{q}_j-\bm{\tilde{q}}_j))}.
		\end{split}
	\end{displaymath}
\end{theorem} 

\begin{proof}
	\begin{displaymath}
		\begin{split}
			& Q_{uni}(y_i|\bm{z}_u)= P_1(k_1)\cdot P_2(k_2|k_2) \cdot P_3(y_i|k_1, k_2)\\
			=&\frac{\Psi_{k_1}' \exp({\bm{z}_u^1}^\top \bm{c}^1_{k_1})}{\sum_{k=1}^K \Psi_{k}'\exp({\bm{z}_u^1}^\top \bm{c}^1_{k})} \cdot
			\frac{\omega_{k_1, k_2}'\exp({\bm{z}_u^2}^\top \bm{c}^2_{k_2})}{\sum_{k=1}^K \omega_{k_1,k}' \exp({\bm{z}_u^2}^\top  \bm{c}^2_{k})} \cdot
			\frac{1}{|\Omega_{k_1, k_2}|}\\
			=&\frac{\exp({\bm{z}_u^1}^\top \bm{c}^1_{k_1}) \exp({\bm{z}_u^2}^\top \bm{c}^2_{k_2})}{\sum_{k=1}^K\sum_{k'=1}^K |\Omega_{k, k'}|\exp({\bm{z}_u^1}^\top \bm{c}^1_{k}) \exp({\bm{z}_u^2}^\top  \bm{c}^2_{k'})}\\
			=&\frac{\exp(\bm{z}_u^\top(\bm{q}_i-\bm{\tilde{q}}_i))} {\sum_{j\in \mathcal{I}}\exp(\bm{z}_u^\top(\bm{q}_j-\bm{\tilde{q}}_j))}.
		\end{split}
	\end{displaymath}
\end{proof}

\begin{theorem}[\textbf{Theorem 5.2}]
	Suppose $P_1(\cdot)$ and $P_2(\cdot|k_1)$ remain the same as that in Theorem 4.1, $P_3(\cdot|k_1,k_2)$ is replaced with a distribution derived from the popularity, i.e. $P_3(y_i|k_1,k_2)=\frac{pop(i)}{\sum_{j \in \Omega_{k_1, k_2}} pop(j)}$ where $pop(i)$ can be any metric of the popularity. Then, the proposal distribution is equivalent to:
	\begin{displaymath}
		Q_{\text{pop}}(y_i|\bm{z}_u)=\frac{\exp(\bm{z}_u^\top(\bm{q}_i-\bm{\tilde{q}}_i) + \log pop(i))} {\sum_{j\in \mathcal{I}}\exp(\bm{z}_u^\top(\bm{q}_j-\bm{\tilde{q}}_j) + \log pop(j))}.
	\end{displaymath}
\end{theorem}
\begin{proof}
	\begin{displaymath}
		\begin{split}
			& Q_{pop}(y_i|\bm{z}_u)= P_1(k_1)\cdot P_2(k_2|k_2) \cdot P_3(y_i|k_1, k_2)\\
			=&\frac{\Psi_{k_1}' \exp({\bm{z}_u^1}^\top \bm{c}^1_{k_1})}{\sum_{k=1}^K \Psi_{k}'\exp({\bm{z}_u^1}^\top \bm{c}^1_{k})} \cdot
			\frac{\omega_{k_1, k_2}'\exp({\bm{z}_u^2}^\top \bm{c}^2_{k_2})}{\sum_{k=1}^K \omega_{k_1,k}' \exp({\bm{z}_u^2}^\top  \bm{c}^2_{k})} \cdot
			\frac{pop(i)}{\sum_{j \in \Omega_{k_1, k_2}} pop(j)}\\
			=&\frac{pop(i)\exp({\bm{z}_u^1}^\top \bm{c}^1_{k_1}) \exp({\bm{z}_u^2}^\top \bm{c}^2_{k_2})}{\sum_{k=1}^K\sum_{k'=1}^K \sum_{j \in \Omega_{k_1, k_2}} pop(j)\exp({\bm{z}_u^1}^\top \bm{c}^1_{k}) \exp({\bm{z}_u^2}^\top  \bm{c}^2_{k'})}\\
			=&\frac{\exp(\bm{z}_u^\top(\bm{q}_i-\bm{\tilde{q}}_i)+\log pop(i))} {\sum_{j\in \mathcal{I}}\exp(\bm{z}_u^\top(\bm{q}_j-\bm{\tilde{q}}_j)+\log pop(j))}.
		\end{split}
	\end{displaymath}
\end{proof}

\begin{theorem}[\textbf{Theorem 5.3}]
	Assuming that the residual embedding $ \Vert \bm{\tilde{q}}_i\Vert \leq C $, the Kullback–Leibler divergence from the softmax distribution $ Q(\bm{y}_{\cdot}|\bm{z}_u)$ to the proposed distribution $Q_{\text{uni}}(\bm{y}_{\cdot}|\bm{z}_u)$ can be bounded from above:
	\begin{displaymath}
		0 < \mathcal{D}_{KL}\left[ Q_{\text{uni}}(\bm{y}_{\cdot}|\bm{z}_u) ||  Q(\bm{y}_{\cdot}|\bm{z}_u) \right] \le 2  C \Vert \bm{z}_u \Vert.
	\end{displaymath}
\end{theorem}

\begin{proof}
	\begin{small}
	\begin{displaymath}
		\begin{aligned}
			& Q(y_i|\bm{z}_u) = \frac{\exp(\bm{z}_u^\top \bm{q}_i)}{\sum_{j\in \mathcal{I}} \exp(\bm{z}_u^\top \bm{q}_j)},\\
			& Q_{\text{uni}}(y_i|\bm{z}_u) =\frac{\exp(\bm{z}_u^\top(\bm{q}_i-\bm{\tilde{q}}_i))} {\sum_{j\in \mathcal{I}}\exp(\bm{z}_u^\top(\bm{q}_j-\bm{\tilde{q}}_j))} ,\\ 
		\end{aligned}
	\end{displaymath}
	\begin{displaymath}
		\begin{aligned}
			&\frac{ Q_{\text{uni}}(y_i|\bm{z}_u) }{ Q(y_i|\bm{z}_u )} = \frac{\sum_{j\in \mathcal{I}} \exp(\bm{z}_u^\top \bm{q}_j)}{ \sum_{k\in \mathcal{I}}\exp(\bm{z}_u^\top(\bm{q}_k-\bm{\tilde{q}}_k)) } \cdot \exp(-\bm{z}_u^\top \bm{\tilde{q}}_i) \\
			= & \sum_{j\in \mathcal{I}} \exp(\bm{z}_u^\top (\bm{q}_j-\bm{\tilde{q}}_j)) \cdot \frac{\exp(\bm{z}_u^\top \bm{\tilde{q}}_j)}{\sum_{k\in \mathcal{I}}\exp(\bm{z}_u^\top(\bm{q}_k-\bm{\tilde{q}}_k))} \cdot \exp(-\bm{z}_u^\top \bm{\tilde{q}}_i)\\
			= & \sum_{j\in \mathcal{I}} \exp(\bm{z}_u^\top (\bm{q}_j-\bm{\tilde{q}}_j)) \cdot \frac{\exp(\bm{z}_u^\top (\bm{\tilde{q}}_j -\bm{\tilde{q}}_i ))}{\sum_{k\in \mathcal{I}}\exp(\bm{z}_u^\top(\bm{q}_k-\bm{\tilde{q}}_k))} \\
			\le & \sum_{j\in \mathcal{I}} \exp(\bm{z}_u^\top (\bm{q}_j-\bm{\tilde{q}}_j)) \cdot \frac{\exp( \vert \bm{z}_u^\top (\bm{\tilde{q}}_j -\bm{\tilde{q}}_i ) \vert)}{\sum_{k\in \mathcal{I}}\exp(\bm{z}_u^\top(\bm{q}_k-\bm{\tilde{q}}_k))} \\
			\le & \sum_{j\in \mathcal{I}} \exp(\bm{z}_u^\top (\bm{q}_j-\bm{\tilde{q}}_j)) \cdot \frac{\exp( \vert \bm{z}_u^\top \bm{\tilde{q}}_j \vert + \vert \bm{z}_u^\top \bm{\tilde{q}}_i  \vert)}{\sum_{k\in \mathcal{I}}\exp(\bm{z}_u^\top(\bm{q}_k-\bm{\tilde{q}}_k))} \\
			= & \sum_{j\in \mathcal{I}} \exp(\bm{z}_u^\top (\bm{q}_j-\bm{\tilde{q}}_j)) \cdot \frac{\exp( 2 C \Vert \bm{z}_u \Vert)}{\sum_{k\in \mathcal{I}}\exp(\bm{z}_u^\top(\bm{q}_k-\bm{\tilde{q}}_k))} \\
			= & \exp (2 C \Vert \bm{z}_u \Vert ),
		\end{aligned}
	\end{displaymath}
	\begin{displaymath}
		\begin{aligned}
			& \mathcal{D}_{KL}\left[ Q_{\text{uni}}(\bm{y}_{\cdot}|\bm{z}_u) ||  Q(\bm{y}_{\cdot}|\bm{z}_u) \right] = \sum_{i\in \mathcal{I }} Q_{\text{uni}}(y_i|\bm{z}_u) \log \frac{ Q_{\text{uni}}(y_i|\bm{z}_u) }{ Q(y_i|\bm{z}_u )} \\
			\le & \sum_{i\in \mathcal{I }} Q_{\text{uni}}(y_i|\bm{z}_u) \log \exp (2 C \Vert \bm{z}_u \Vert ) \\
			= & \sum_{i\in \mathcal{I }} Q_{\text{uni}}(y_i|\bm{z}_u) 2 C \Vert \bm{z}_u \Vert \\
			= & 2 C \Vert \bm{z}_u \Vert  \sum_{i\in \mathcal{I }} Q_{\text{uni}}(y_i|\bm{z}_u)\\
			= & 2 C \Vert \bm{z}_u \Vert. 
		\end{aligned}
	\end{displaymath}
\end{small}
$ \mathcal{D}_{KL}\left[ Q_{\text{uni}}(\bm{y}_{\cdot}|\bm{z}_u) ||  Q(\bm{y}_{\cdot}|\bm{z}_u) \right] > 0$ holds due to the non-negativity of the Kullback–Leibler divergence. 
\end{proof}

\begin{theorem}[\textbf{Theorem 5.4}]\label{proof:kl_midx_pop}
	Assuming that the residual embedding $ ||\tilde{\bm{q}}_i|| \leq C $, the Kullback–Leibler divergence from the softmax distribution $Q(\bm{y}_{\cdot}|\bm{z}_u)$ to the proposed distribution $Q_{\text{pop}}(\bm{y}_{\cdot}|\bm{z}_u)$ can be bounded from above:
	\begin{displaymath}
		0 < \mathcal{D}_{KL}\left[ Q_{\text{pop}}(\bm{y}_{\cdot}|\bm{z}_u) ||  Q(\bm{y}_{\cdot}|\bm{z}_u) \right] \le 2  C ||\bm{z}_u|| + \log \frac{\max pop(\cdot)}{\min pop(\cdot)}.
	\end{displaymath}
\end{theorem}

\begin{proof}
	\begin{small}
	\begin{displaymath}
		\begin{aligned}
			& \text{Denote }\mathcal{E}(i)= \exp(\bm{z}_u^\top(\bm{q}_i-\bm{\tilde{q}}_i) + \log pop(i)),\\
			& \frac{ Q_{\text{pop}}(y_i|\bm{z}_u) }{ Q(y_i|\bm{z}_u )}  = \frac{ \sum_{j\in \mathcal{I}} \exp(\bm{z}_u^\top \bm{q}_j) \cdot \exp(-\bm{z}_u^\top \bm{\tilde{q}}_i + \log pop(i))}{ \sum_{k\in \mathcal{I}}\exp(\bm{z}_u^\top(\bm{q}_k-\bm{\tilde{q}}_k) + \log pop(k)) } \\
			= & \frac{ \sum_{j\in \mathcal{I}} \mathcal{E}(j)}{ \sum_{k\in \mathcal{I}}\mathcal{E}(k)} \cdot \frac{ \exp( \bm{z}_u^\top \bm{\tilde{q}}_j - \log pop(j))}{ \exp (\bm{z}_u^\top \bm{\tilde{q}}_i - \log pop(i))} \\
			= & \frac{ \sum_{j\in \mathcal{I}} \mathcal{E}(j)}{ \sum_{k\in \mathcal{I}}\mathcal{E}(k)} \cdot \exp( \bm{z}_u^\top ( \bm{\tilde{q}}_j - \bm{\tilde{q}}_i) ) \cdot \frac{pop(i)}{pop(j)} \\
			\le & \frac{ \sum_{j\in \mathcal{I}} \mathcal{E}(j)}{ \sum_{k\in \mathcal{I}}\mathcal{E}(k)} \cdot \exp( 2C \Vert \bm{z}_u \Vert) \cdot \frac{\max pop(\cdot)}{\min pop(\cdot)} \\
			= & \frac{\max pop(\cdot)}{\min pop(\cdot)} \cdot \exp( 2C \Vert \bm{z}_u \Vert),\\
		\end{aligned}
	\end{displaymath}
	\begin{displaymath}
		\begin{aligned}
			& \mathcal{D}_{KL}\left[ Q_{\text{pop}}(\bm{y}_{\cdot}|\bm{z}_u) ||  Q(\bm{y}_{\cdot}|\bm{z}_u) \right] = \sum_{i\in \mathcal{I }} Q_{\text{pop}}(y_i|\bm{z}_u) \log \frac{ Q_{\text{pop}}(y_i|\bm{z}_u) }{ Q(y_i|\bm{z}_u )} \\
			\le & \sum_{i\in \mathcal{I }} Q_{\text{pop}}(y_i|\bm{z}_u) \log \exp (2 C \Vert \bm{z}_u \Vert + \log \frac{\max pop(\cdot)}{\min pop(\cdot)} ) \\
			= &  \log \exp (2 C \Vert \bm{z}_u \Vert + \log \frac{\max pop(\cdot)}{\min pop(\cdot)} ) \sum_{i\in \mathcal{I }} Q_{\text{pop}}(y_i|\bm{z}_u)\\
			= & 2 C \Vert \bm{z}_u \Vert + \log \frac{\max pop(\cdot)}{\min pop(\cdot)}.
		\end{aligned}
	\end{displaymath}
	\end{small}
\end{proof}

\end{document}